\documentclass{scrartcl}

\RequirePackage[colorlinks,citecolor=blue,urlcolor=blue]{hyperref}
\usepackage{amsmath}
\usepackage{amsfonts}
\usepackage{dsfont}
\usepackage[T1]{fontenc}
\usepackage[latin1]{inputenc}
\usepackage{amsthm}
\usepackage{multirow}
\usepackage{hhline}
\usepackage{graphicx}
\usepackage{bm}
\usepackage{listings}

\usepackage{tikz}
\usetikzlibrary{positioning}
\usetikzlibrary{shapes,snakes}
\usetikzlibrary{shapes,arrows}

\usepackage{booktabs,color,epsfig,url}

\makeatletter
\DeclareOldFontCommand{\rm}{\normalfont\rmfamily}{\mathrm}
\DeclareOldFontCommand{\sf}{\normalfont\sffamily}{\mathsf}
\DeclareOldFontCommand{\tt}{\normalfont\ttfamily}{\mathtt}
\DeclareOldFontCommand{\bf}{\normalfont\bfseries}{\mathbf}
\DeclareOldFontCommand{\it}{\normalfont\itshape}{\mathit}
\DeclareOldFontCommand{\sl}{\normalfont\slshape}{\@nomath\sl}
\DeclareOldFontCommand{\sc}{\normalfont\scshape}{\@nomath\sc}
\makeatother

\newcommand{\x}{\bold{x}}

\newcommand{\F}{{\mathcal{F}}}
\newcommand{\N}{\mathbb{N}}

\newcommand{\R}{\mathbb{R}}
\newcommand{\Rd}{\mathbb{R}^d}

\newcommand{\bi}{\boldsymbol{i}}

\newcommand{\bll}{\bm{\mathbf{\ell}}}

\newtheorem{theorem}{Theorem}

\newtheorem{lemma}{Lemma}

\newtheorem{definition}{Definition}
\theoremstyle{definition}
\newtheorem{remark}{Remark}

\renewcommand{\P}{{\mathcal{P}}}

\allowdisplaybreaks

\begin{document}

\begin{center}
{\LARGE \textbf{
Approximating smooth functions by deep neural networks with sigmoid activation function}}
\footnote{
Running title: {Approximation properties of deep neural networks}}
\vspace{0.5cm}

Sophie Langer\footnote{
Corresponding author. Tel: +49-6151-16-23371} 

{\textit{Fachbereich Mathematik, Technische Universit\"at Darmstadt, \\
Schlossgartenstr. 7, 64289 Darmstadt, Germany,
\\
email: langer@mathematik.tu-darmstadt.de}}

\end{center}
\vspace{0.5cm}

\begin{center}
October 8, 2018
\end{center}
\vspace{0.5cm}

\noindent
{\textbf{Abstract}}\\
We study the power of deep neural networks (DNNs) with sigmoid activation function. 
Recently, it was shown that DNNs approximate any $d$-dimensional, smooth function on a compact set
with a rate of order $W^{-p/d}$, where $W$ is the number of nonzero weights in the network and $p$ is the smoothness of the function. Unfortunately,  these rates only hold for a special class of sparsely connected DNNs. We ask ourselves if we can show the same approximation rate for a simpler and more general class, i.e., DNNs which are only defined by its width and depth. In this article we show that DNNs with fixed depth and a width of order $M^d$ achieve an approximation rate of $M^{-2p}$. As a conclusion we quantitatively characterize the approximation power of DNNs in terms of the overall weights $W_0$ in the network and show an approximation rate of $W_0^{-p/d}$. 
 This more general result finally helps us to understand which network topology guarantees a special target accuracy.

\vspace*{0.2cm}


\vspace*{0.2cm}


\section{Introduction}
The outstanding performance of deep neural networks (DNNs) in the application on various tasks like
pattern recognition (\cite{LBH15, Sch15}), speech recognition (\cite{H12}) and game intelligence (\cite{SHMG16}) 
have led to an increasing interest in the literature in showing good theoretical properties of these networks. Till now there is 
still a lack of mathematical understanding of why neural networks with many hidden layers, also known as \textit{Deep Learning}, are so succesful in practice. Recently, there are different research topics to also prove the power of DNNs from a theoretical point of view. One key question we search an answer for is the approximation capacity of DNNs, i.e., we are interested in how multivariate functions can be approximated by DNNs. Several results already contributed to this: 
The analysis of neural networks with one hidden layer resulted in the so--called universal approximation 
theorem, stating that any continuous function can be approximated arbitrarily well by 
single--hidden--layer networks provided the number of neurons is large enough (see, e.g., \cite{C89, HSW89}).
In case that the regarded function is differentiable, it could also been shown that the networks 
are able to approximate the first order derivative of the function (\cite{M13}). Nevertheless, the number of 
required neurons per layer have to be very large in some cases. An overview of approximation results by shallow neural networks is given in \cite{ScTs98}. \cite{MLP16} could show the same approximation result for deep neural networks as for shallow ones for a class of compositional functions, with the main advantage that in case of deep networks a lower number of training parameters is needed for the same degree of accuracy. \cite{ES15} showed, that three--layer networks with a small number of parameters are as efficient as really large two--layer networks. \cite{T15} stated examples of functions that cannot be efficiently represented by shallow neural networks but by deep ones. General approximation results concerning multilayer neural networks were presented in \cite{F89} for continuous functions and in \cite{N99} for functions and their derivatives. For smooth activation functions satisfying $\lim_{x \to - \infty} \sigma(x) = 0$ and $\lim_{x \to \infty} \sigma(x) =1$ and some further properties \cite{BK17} could show, that networks with two hidden layers are able to approximate any smooth function with an approximation error of size $W^{-p/d}$, where $W$ is the number of nonzero weights in the network (see Theorem 2 in \cite{BK17}). The functions under study are $(p,C)$--smooth, i.e., they fulfill the following definition:
\begin{definition}
\label{intde2} 
  Let $p=q+s$ for some $q \in \N_0$ and $0< s \leq 1$.
A function $m:\R^d \rightarrow \R$ is called
$(p,C)$-smooth, if for every $\bm{\alpha}=(\alpha_1, \dots, \alpha_d) \in
\N_0^d$
with $\sum_{j=1}^d \alpha_j = q$ the partial derivative
$\partial^q m/(\partial x_1^{\alpha_1}
\dots
\partial x_d^{\alpha_d}
)$
exists and satisfies
\[
\left|
\frac{
\partial^q m
}{
\partial x_1^{\alpha_1}
\dots
\partial x_d^{\alpha_d}
}
(\x)
-
\frac{
\partial^q m
}{
\partial x_1^{\alpha_1}
\dots
\partial x_d^{\alpha_d}
}
(\bold{z})
\right|
\leq
C
\| \x-\bold{z} \|^s
\]
for all $\x,\bold{z} \in \R^d$, where $\Vert\cdot\Vert$ denotes the Euclidean norm.
\end{definition}
A similar result for networks with ReLU activation functions was presented by \cite{Sch17} (see Theorem 5 in \cite{Sch17}). Unfortunately, both results only hold for a special class of sparsely connected DNNs, where the number of nonzero weights $W$ is much smaller than the number of overall weights $W_0$. A simpler class of DNNs, so-called fully connected DNNs, was analyzed by \cite{KL20} and \cite{YZ19}. Those networks are only defined by its width and depth and do not depend on a further sparsity constraint. \cite{KL20} divided their work in two different cases: 1) DNNs with varying width and logarithmic depth 2) DNNs with varying depth and fixed width. For the first case they derived an approximation rate of $W_0^{-p/d}$ and in their second case they even improved this rate to $W_0^{-2p/d}$. The second case was also shown in \cite{YZ19}.  \cite{LS20} went one step futher and presented a result where both width and depth are varied simultaneously. Beside \cite {BK17}, all the above mentioned results focus on the ReLU activation function. While \cite{BK17} already shows an approximation error of order $W^{-p/d}$ for sigmoidal networks, this result only holds for a special class of DNNs. In this article we define our DNNs only by its depth and width and show that for a fixed-depth DNN with bounded weights and width of order $M^d$ we can achieve an approximation rate of $M^{-2p}$. This in turn generalizes the result of \cite{BK17}, since this also proves an rate of $W_0^{-p/d}$ in terms of the overall number of weights $W_0$. The topology
of our networks is clearly defined, and can therefore been seen as a guideline of how 
the network architecture have to be chosen to receive 
good approximation results for different function classes. 
In the proofs we generalize the techniques from the proof of the approximation results in \cite{KL20} to smooth activation function. This enables us to derive the same
approximation results as in \cite{BK17} but with respect to the supremum norm on a cube. 
\\
\\
Our class of DNNs is defined as follows: As an activation function $\sigma: \mathbb{R} \to [0,1]$ we choose the sigmoid activation function
\begin{equation}
  \label{inteq4}
\sigma(x)=\frac{1}{1+\exp(-x)} \quad (x \in \R).
\end{equation} 
The network architecture $(L, \textbf{k})$ depends on a positive integer $L$ called the \textit{number of hidden layers} or \textit{depth} and a \textit{width vector} $\textbf{k} = (k_1, \dots, k_{L}) \in \mathbb{N}^{L}$ that describes the number of neurons in the first, second, $\dots$, $L$-th hidden layer. A feedforward DNN with network architecture $(L, \textbf{k})$ and sigmoid activation function $\sigma$ is a real-valued function defined on $\mathbb{R}^d$ of the form
\begin{equation}\label{inteq1}
f(\bold{x}) = \sum_{i=1}^{k_L} c_{1,i}^{(L)} f_i^{(L)}(x) + c_{1,0}^{(L)}, 
\end{equation}
for some $c_{1,0}^{(L)}, \dots, c_{1,k_L}^{(L)} \in \mathbb{R}$ and for $f_i^{(L)}$'s recursively defined by
\begin{equation}
  \label{inteq2}
f_i^{(s)}(\bold{x}) = \sigma\left(\sum_{j=1}^{k_{s-1}} c_{i,j}^{(s-1)}f_j^{(s-1)}(x) + c_{i,0}^{(s-1)} \right)
\end{equation}
for some $c_{i,0}^{(s-1)}, \dots, c_{i, k_{s-1}}^{(s-1)} \in \mathbb{R}$,
$s\in \{2, \dots, L\}$
and
\begin{equation}
  \label{inteq3}
f_i^{(1)}(\bold{x}) = \sigma \left(\sum_{j=1}^d c_{i,j}^{(0)} x^{(j)} + c_{i,0}^{(0)} \right)
\end{equation}
for some $c_{i,0}^{(0)}, \dots, c_{i,d}^{(0)} \in \mathbb{R}$. 
The space of DNNs with depth
$L$, width $r$ and all coefficients bounded by $\alpha$ is defined by
\begin{align}\label{F}
  \mathcal{F}(L, r, \alpha) = \{ &f \, : \,  \text{$f$ is of the form } \eqref{inteq1}
  \text{ with }
k_1=k_2=\dots=k_L=r \notag \\
& \text{and} \
|c_{i,j}^{(\ell)}| \leq \alpha \ \text{for all} \ i,j,\ell
\}.
\end{align}
Here it is easy to see that DNNs of the class $\mathcal{F}(L,r,\alpha)$ are not restricted by a further sparsity constraint and are only defined by the depth $L$, width $r$ and a bound $\alpha$ for the weights in the network.

\subsection{Notation}
Throughout the paper, the following notation is used:
The sets of natural numbers, natural numbers including $0$ and real numbers
are denoted by $\N$, $\N_0$ and $\R$, respectively. 
For $z \in \R$ we set $z_+=\max\{z,0\}$.
Vectors are denoted by bold letters, e.g., $\bold{x} = (x^{(1)}, \dots, x^{(d)})^T$. We define
$\bold{1}=(1, \dots, 1)^T$ and $\bold{0} = (0, \dots, 0)^T$. A $d$-dimensional multi-index is a $d$-dimensional vector 
$\bold{j} = (j^{(1)}, \dots, j^{(d)})^T \in \N_0^d$. As usual, we define
\begin{align*}
&\|\bold{j}\|_1 = j^{(1)} + \dots + j^{(d)}, \quad \bold{x}^{\bold{j}} = (x^{(1)})^{j^{(1)}} \cdots (x^{(d)})^{j^{(d)}}, \\ & \bold{j}! = j^{(1)}! \cdots j^{(d)}!,\, \quad \partial^{\bold{j}} =\frac{\partial^{j^{(1)}}}{\partial(x^{(1)})^{j^{(1)}}} \cdots \frac{\partial^{j^{(d)}}}{\partial(x^{(d)})^{j^{(d)}}}.
\end{align*}
Let $D \subseteq \R^d$ and let $f:\R^d \rightarrow \R$ be a real-valued
function defined on $\R^d$.
The Euclidean, the supremum and the $L_1$ norms of $\x \in \Rd$ 
are denoted by $\|\x\|$, $\|\x\|_\infty$ and $\|\x\|_1$, respectively. 
For $f:\R^d \rightarrow \R$
\[
\|f\|_\infty = \sup_{\x \in \R^d} |f(\x)|
\]
is its supremum norm, and the supremum norm of $f$
on a set $A \subseteq \R^d$ is denoted by
\[
\|f\|_{\infty,A} = \sup_{\x \in A} |f(\x)|.
\]
Furthermore we define $\|\cdot\|_{C^q(A)}$ of the smooth function space $C^q(A)$ by
\[
\|f\|_{C^q(A)} := \max\left\{\|\partial^{\bold{j}} f\|_{\infty, A}: \|\bold{j}\|_1 \leq q, \bold{j} \in \N^d\right\}
\]
for any $f \in C^q(A)$.

\subsection{Outline}
The main body of this article consists of Section \ref{se2}, \ref{se3} and \ref{se4}.
Section \ref{se2} presents our main result on the approximation of 
$(p,C)$--smooth functions by DNNs. Section \ref{se3} deals with the approximation properties of our neural networks for simpler function classes.
In Section \ref{se4} the proof of our main result is given.

\section{Main result}
\label{se2}
Our main result, concerning the approximation of $(p,C)$--smooth functions by DNNs with a fixed depth and a bound for all coefficients of the network, is the following: 

\begin{theorem}
  \label{th1}
  Let $1 \leq a < \infty$. Let $p=q+s$ for some $q \in \mathbb{N}_0$ and $s \in (0,1]$,
  let $C \geq 1$. Let $f: \mathbb{R}^d \to \mathbb{R}$ be a $(p, C)$-smooth function, which satisfies
  \begin{equation}\label{th2eq1}
  \|f\|_{C^{q}([-2a, 2a]^d)} \leq c_1
  \end{equation}
  for some constant $c_1 > 0$.
 Let $\sigma : \mathbb{R} \to [0, 1]$ be the sigmoid activation function \linebreak $\sigma(x) = 1/(1+\exp(-x))$. For any $M \in \mathbb{N}$ sufficiently large (independent of the size of $a$, but
     \begin{align*}
      M^{2p} \geq
     2c_{2} \left(\max\left\{a, \|f\|_{C^q([-a,a]^d)}\right\}\right)^{5q+3} \ \mbox{and} \ M^{2p} \geq  \max\left\{c_{3}, 2^d, 12d\right\} 
    \end{align*}
     must hold) we set
     \begin{itemize}
\item[\rm{(i)}] $L = 8+\lceil \log_2(\max\{d, q+1\})\rceil$
\item[\rm{(ii)}] $r = 2^d\left(\max\left\{\left(\binom{d+q}{d} + d\right)M^d (2+2d)+d, 4(q+1) \binom{d+q}{d}\right\} +M^d(2d+2)+12d\right)$
\item[\rm{(iii)}] $\alpha = c_4 \left(\max\left\{a, \|f\|_{C^q([-a,a]^d)}\right\}\right)^{12} e^{6 \times 2^{2(d+1)+1}ad} M^{10p+2d+10}.$
\end{itemize}
Then a neural network $f_{net}$ of the network class $\mathcal{F}\left(L, r, \alpha\right)$ exists such that
\begin{equation*}
\|f_{net}-f\|_{\infty, [-a,a]^d} \leq \frac{c_5\left(\max\left\{a, \|f\|_{C^q([-a,a]^d)}\right\}\right)^{5q+3}}{M^{2p}}
\end{equation*}
holds for a constant $c_5 >0$.
\end{theorem}
\begin{remark}
\autoref{th1} shows, that there exists a DNN with 
\begin{align*}
W_0 = (d+1)r+ (L-1)r(r+1)+(r+1) = c_6M^{2d}
\end{align*}
weights that achieves the supremum norm approximation rate $W^{-p/d} = M^{-2p}$.
\end{remark}
\section{Approximating different function classes by DNNs}
\label{se3}
\subsection{Some computing operations for DNNs}
In the following approximation results we often combine subnetworks for simpler tasks to construct networks for more complex tasks, i.e., the combination of networks approximating a multiplication finally leads to a network approximating monomials. To get a better understanding of how our networks are put together, we give a short overview of two different computing operations.\\
\\
\textit{Combined neural network:} Let $f \in \mathcal{F}(L_f, r_f, \alpha_f)$ and $g \in \mathcal{F}(L_g, r_g, \alpha_g)$ with $L_f, L_g, r_f, \linebreak r_g \in \N$ and $\alpha_f, \alpha_g \in \R$, then we call $f \circ g$ the \textit{combined network}, which is contained in the network class $\mathcal{F}(L_f+L_g, \max\{r_f, r_g\}, \max\{\alpha_f,1\}\max\{\alpha_g, 1\})$. Here we ``melt'' the output layer of $g$ with the input layer of $f$. This is why we multiply the bound of the weights of both networks in the combined network.\\
\textit{Parallelized neural network:} Let $f \in \mathcal{F}(L, r_f, \alpha_f)$ and $g \in \mathcal{F}(L, r_g, \alpha_g)$ be two neural networks with the same number of hidden layers $L \in \N$. Then we call $(f,g)$ the \textit{parallelized network}, which computes $f$ and $g$ parallel in a joint network. Thus $(f,g)$ is a network with depth $L$, width $r_f + r_g$ and a bound of the coefficients of size $\max\{\alpha_f, \alpha_g\}$. 
%
\subsection{Approximation of identity and multiplication}
The following lemma presents a neural network that approximates the identity function. 
\begin{lemma}
  \label{le1}
Let $\sigma:\R \rightarrow \R$ be a function, let $R \geq 1$ and $a>0$. Assume that $\sigma$ is two times continuously differentiable and
let
$t_{\sigma,id} \in \R$ be such that $\sigma^\prime(t_{\sigma,id}) \neq
0$.
Then there exists a neural network $f_{id} \in \F(1,1,c_{7} R)$
that satisfies for any $x \in [-a,a]$:
\[
| f_{id}(x)-x|
\leq
\frac{
\| \sigma^{\prime \prime}\|_{\infty}
}{
2 |\sigma^\prime(t_{\sigma,id})|
}
\frac{a^2}{R}.
\]
\end{lemma}

\begin{proof}[\rm{\textbf{Proof of \autoref{le1}.}}]
   The result follows in a straightforward way from the proof of
   Theorem 2 in \cite{ScTs98}, cf., e.g.,
   Lemma 1 in \cite{KL20a}.
   \end{proof}
   In the sequel we will use the abbreviations
\begin{align*}
f_{id}(\bold{z}) = \left(f_{id}\left(z^{(1)}\right), \dots, f_{id}\left(z^{(d)}\right)\right), \quad \bold{z} \in \Rd
\end{align*}
and
\begin{alignat*}{3}
f_{id}^0(\bold{z}) = \bold{z}, \quad f_{id}^{t+1}(\bold{z}) = f_{id}\left(f_{id}^t(\bold{z})\right) , \quad t \in \N_0, \bold{z} \in \R^d.
\end{alignat*}

The next lemma presents a network which returns approximately $xy$ given the input $x$ and $y$.
\begin{lemma}
\label{le2}
Let $\sigma: \mathbb{R} \to [0,1]$ be the sigmoid activation function $\sigma(x) = 1/(1+\exp(-x))$. Then for any $R \geq 1$ and any $a >0$ there exists a neural network $f_{mult} \in \F(1, 4, 3R^2)$ 
that satisfies for any $x,y \in [-a,a]$:
\begin{equation*}
|f_{mult}(x,y) - xy| \leq
75 \|
\sigma^{\prime \prime \prime}
\|_{\infty}
\frac{a^3}{R}.
\end{equation*}
\end{lemma}

\begin{proof}[\rm{\textbf{Proof of \autoref{le2}}}]
See Lemma 2 in \cite{KL20a}.
\end{proof}

%
%

\subsection{Approximation of multivariate polynomials}
Let $\mathcal{P}_N$ be the linear span of all monomials of the form
\begin{equation*}
\prod_{k=1}^d(x^{(k)})^{r_k}
\end{equation*}
for some $r_1,\dots ,r_d \in \N_0$, $r_1+\dots+r_d \leq N$. Then $\mathcal{P}_N$ is a linear vector space of functions
of dimension
\begin{equation*}
   \dim \mathcal{P}_N = \left|\left\{(r_1, \dots, r_d) \in \N_0^d: r_1+ \dots +r_d \leq N \right\}\right| = \binom{d+N}{d}.
\end{equation*}
Our next lemma presents a neural network that approximates multivariate polynomials multiplied by an additional factor. This modified form of multivariate polynomials is later needed in the construction of the network for the main result.

\begin{lemma}\label{le3}
  Let $N \in \mathbb{N}_0$ and let $p \in \mathcal{P}_N$. Let $m_1, \dots, m_{\binom{d+N}{d}}$ denote all monomials in $\mathcal{P}_N$.
  For $p \in \mathcal{P}_N$
  define $r_1, \dots, r_{\binom{d+N}{d}} \in \R$ by 
\begin{equation} \label{le4p}
p\left(\x, y_1, \dots, y_{\binom{d+N}{d}}\right) = \sum_{i=1}^{\binom{d+N}{d}} r_i y_i m_i(\x), \ \ \ \ \ \x \in [-a,a]^d, y_i \in [-a,a]
\end{equation}
and set $\bar{r}(p) = \max_{i \in \left\{1, \dots, \binom{d+N}{d}\right\}}|r_i|$. Let $\sigma: \R \to [0,1]$ be the sigmoid activation function $1/(1+\exp(-x))$ . Then for any
$a \geq 1$
 and any
 \begin{align}
 \label{le3eq1}
 R \geq
\max\left\{75\|
\sigma^{\prime \prime \prime}
\|_{\infty}4^{3(N+1)} a^{3(N+1)}, 1\right\}
 \end{align}
 a neural network 
$f_{p} \in \mathcal{F}(L,r, \alpha_p)$
with $L = \lceil\log_2 (N+1) \rceil $, $r = 4(N+1) \binom{d+N}{d}$ and $ \alpha_p = 9\max\{\bar{r}(p),1\}R^4$ exists such that
\begin{equation*}
  \left|f_{p}\left(\x, y_1, \dots, y_{\binom{d+N}{d}}\right) - p\left(\x, y_1, \dots, y_{\binom{d+N}{d}}\right)\right| \leq c_{9} \bar{r}(p) \frac{a^{5N+3}}{R}
\end{equation*}
holds for all $\x \in [-a, a]^d$ and a constant
\begin{align*}
c_{9} = \binom{d+N}{d}150 \|
\sigma^{\prime \prime \prime}
\|_{\infty} N 4^{5N+3}.
\end{align*}
\end{lemma}

\begin{proof}[\rm{\textbf{Proof of \autoref{le3}}}]
This proof follows in a straightforward modification from the proof of Lemma 5 in the Supplement of \cite{KL20}. A complete proof can be found in the Supplement. 
\end{proof}

\subsection{Approximation of the indicator function}
The following lemma presents a neural network that approximates the multidimensional indicator function and the multidimensional indicator function multiplied by a further value. 
\begin{lemma}
\label{le4}
Let $\sigma: \R \to [0,1]$ be the sigmoid activation function $\sigma(x) = 1/(1+\exp(-x))$, $\epsilon \in (0,1)$ and $\delta > 0$. Let $\bold{a}, \bold{b} \in \Rd$ with
\begin{align*}
b^{(i)} - a^{(i)} \geq 2 \delta \ \mbox{for all} \ i \in \{1, \dots, d\}, 
\end{align*}
let $K=[\bold{a}, \bold{b}]$
and 
\begin{align*}
K_{\delta} = \left\{\x \in K: x^{(i)} \notin [a^{(i)} - \delta, a^{(i)} + \delta] \cup [b^{(i)} - \delta, b^{(i)} + \delta] \ \mbox{for all} \ i \in \{1, \dots, d\}\right\}.
\end{align*}
Set $B_1 = (8/d) \ln((1/\epsilon) -1)$ and $B_2 = (1/\delta) \ln(3)$.
\\
a) Then the network
\begin{align*}
f_{ind, [\bold{a}, \bold{b})}(\x) &= \sigma\left(-B_1\left(\sum_{i=1}^d \bigg(\sigma\left(B_2 \left(a^{(i)} - x^{(i)}\right)\right) + \sigma\left(B_2 \left(x^{(i)}-b^{(i)}\right)\right)\bigg) - \frac{5}{8}d\right)\right)\\ 
&\in \mathcal{F}(2, 2d, \alpha_{ind, [\bold{a}, \bold{b})})
\end{align*}
with
\begin{align*}
\alpha_{ind, [\bold{a}, \bold{b})}= \max\left\{5\ln((1/\epsilon)-1), \max_{i \in \{1, \dots, d\}} \left|b^{(i)}\right| B_2 \right\}
\end{align*}
satisfies
\begin{align*}
\left|\mathds{1}_{[\bold{a}, \bold{b})}(\x) - f_{ind, [\bold{a}, \bold{b})}(\x)\right| \leq \epsilon 
\end{align*}
for $\x \in K_{\delta}$ and 
\begin{align*}
f_{ind, [\bold{a}, \bold{b})}(\x) \in [0,1]
\end{align*}
for $\x \in \Rd$.
\\
b) Let $R \in \N$ and  $|s| \leq R$. Let $f_{id}$ be the network of \autoref{le1} and $f_{mult}$ be the network of \autoref{le2}. Then the network
\begin{align*}
f_{test}(\x, \bold{a}, \bold{b}, s) &= f_{mult}(f_{id}^2(s), f_{ind, [\bold{a}, \bold{b})}(\x)) \in \mathcal{F}\left(3, 2+2d, c_{10} \max\left\{\frac{R^5}{\epsilon^2}, \frac{1}{\delta}\right\}\right)
\end{align*}
satisfies
\begin{align*}
  \left|f_{test}(\x, \bold{a}, \bold{b}, s)-
  s \mathds{1}_{[\bold{a}, \bold{b})}(\x)\right| \leq \epsilon
\end{align*}
for $\x \in K_{\delta}$ and 
\begin{align*}
  \left|f_{test}(\bold{x}, \bold{a}, \bold{b}, s) -
  s  \mathds{1}_{[\bold{a}, \bold{b})}(\x)\right| \leq |s|
\end{align*}
for $\x \in \Rd$.
\end{lemma}

\begin{proof}[\rm{\textbf{Proof of \autoref{le4}}}]
A similar result as \autoref{le4} a) can be found in Lemma 7 in \cite{KK20}. \autoref{le4} b) follows by combining the network of \autoref{le4} a) with $f_{id}$ of \autoref{le1} and multiply those two networks with $f_{mult}$ of \autoref{le2}. A complete proof can be found in the Supplement.
\end{proof}

\section{Proof of the main result}
\label{se4}
\subsection{Idea of the proof of Theorem 1} 
The proof of \autoref{th1} is based on the same idea as described in Theorem 2 a) in \cite{KL20}. In the proof
we will approximate a piecewise Taylor polynomial.
To define this piecewise Taylor polynomial, we partition  $[-a,a)^d$
  into $M^d$ and $M^{2d}$ half-open
  equivolume cubes of the form
  \[
[\bold{\tilde{a}},\bold{\tilde{b}})=[\tilde{a}^{(1)},\tilde{b}^{(1)}) \times \dots \times [\tilde{a}^{(d)},\tilde{b}^{(d)}), \quad \bold{\tilde{a}}, \bold{\tilde{b}} \in \Rd,
  \]
respectively.
Let 
\begin{align}
\label{partition}
\mathcal{P}_1=\{C_{k,1}\}_{k \in \{1, \dots, M^d\}} \ \mbox{and} \ \mathcal{P}_2=\{C_{j,2}\}_{j \in \{1, \dots, M^{2d}\}}
\end{align}
be the corresponding partitions. If $\P$ is a partition of $[-a,a)^d$
and $\x \in [-a,a)^d$, then we denote the cube $C \in \P$, 
which satisfies $\x \in C$, by $C_\P (\x)$. If $C$ is
a cube we denote the "bottom left" corner of $C$
by $C_{left}$. Therefore, each cube $C$
with side length $s$
(which is half-open as the cubes in $\P_1$ and $\P_2$)
can be written as a polytope defined by 
\begin{align*}
-x^{(j)} + C_{left}^{(j)} \leq 0 \ \mbox{and} \ x^{(j)} - C_{left}^{(j)}-s < 0 \quad (j \in \{1, \dots, d\}).
\end{align*}
Furthermore, we describe by $C_{\delta}^0 \subset C$ the cube, which contains all $\x \in C$ that lie with a distance of at least $\delta$ to the boundaries of $C$, i.e., $C_{\delta}^0$ is a polytope defined by
\begin{align*}
-x^{(j)} + C_{left}^{(j)} \leq - \delta \ \mbox{and} \ x^{(j)} - C_{left}^{(j)}-s < -\delta \quad (j \in \{1, \dots, d\}).
\end{align*}
For each $i \in \{1, \dots, M^d\}$ we denote 
those cubes of $\mathcal{P}_2$ that are contained in $C_{i,1}$
by $\tilde{C}_{1, i}, \dots, \tilde{C}_{M^d, i}$.
Here we order the cubes in such a way that
\begin{align}
\label{tildec}
(\tilde{C}_{k, i})_{left} = (C_{i,1})_{left} + \bold{v}_k,
\end{align}
holds for all $k \in \{1, \dots, M^d\},i \in \{1, \dots, M^d\}$ and for some vector $\bold{v}_k=(v_k^{(1)}, \dots, v_k^{(d)})$ with entries in $\{0, 2a/M^2, \dots, (M-1) 2a/M^2\}$. The vector $\bold{v}_k$
describes the position of
$(\tilde{C}_{k, i})_{left}$ relative to $(C_{i,1})_{left}$,
and we order the above cubes such that
this position is independent of $i$. 
\\
\\
The following lemma helps us to approximate our function by a local Taylor polynomial.
\begin{lemma}
\label{le5}
Let $p=q+s$ for some $q \in \N_0$ and $s \in (0,1]$, and let $C > 0$. Let $f: \Rd \to \R$ be a $(p,C)$-smooth function, let $\x_0 \in \Rd$ and let $T_{f,q,\x_0}$ be the Taylor polynomial of total degree $q$ around $\x_0$ defined by
  \begin{eqnarray*}
T_{f,q,\x_0}(\x) &=& \sum_{\bold{j} \in \N_0^d: \|\bold{j}\|_1 \leq q}  (\partial^{\bold{j}} f)(\bold{x}_0) \frac{\left(\x- \x_0\right)^{\bold{j}}}{\bold{j}!}.
  \end{eqnarray*}
Then for any $\x \in \Rd$ 
\begin{align*}
\left|f(\x) - T_{f,q,\x_0}(\x)\right| \leq c_{11} C \Vert \x - \x_0 \Vert^p
\end{align*}
holds for a constant $c_{11}=c_{11}(q,d)$ depending only on $q$ and $d$.
\end{lemma}
\begin{proof}[\rm{\textbf{Proof of \autoref{le5}}}]
See Lemma 1 in \cite{Ko14}.
\end{proof}

From \autoref{le5} we can conclude that the piecewise Taylor polynomial
\begin{align}
\label{taylor}
T_{f,q,(C_{\P_2}(\x))_{left}}(\x)
=
\sum_{k \in \{1, \dots, M^d\}, i \in \{1, \dots, M^d\} } T_{f,q,(\tilde{C}_{k,i})_{left}}(\x)\mathds{1}_{\tilde{C}_{k,\bi}}(\x) 
\end{align}
satisfies
\[
  \|
f - T_{f,q,(C_{\P_2}(\x))_{left}}
\|_{\infty, [-a,a)^d}
\leq
\frac{c_{11} (2ad)^{2p}C}{M^{2p}}.
\]
This means, that a network approximating the above defined piecewise Taylor polynomial with an error of size $M^{-2p}$ also approximates a $(p,C)$-smooth function $f$ with the same error. \\
\\
In the following we divide our proof in \textit{four} key steps:
\begin{enumerate}
\item[1.] Compute $T_{f,q,(C_{\P_2}(\bold{x}))_{left}}(\bold{x})$ by using recursively defined functions.
\item[2.] Approximate the recursive functions by neural networks. The resulting network is a good approximation for $f(\bold{x})$ in case that
\[\bold{x} \in \bigcup_{k \in \{1, \dots, M^{2d}\}} (C_{k,2})_{1/M^{2p+2}}^0.
\]
\item[3.] Construct a neural network to approximate $w_{\P_2}(\bold{x})f(\bold{x})$ in supremum norm, where
\begin{equation}
\label{wp2}
w_{\P_2}(\bold{x}) = \prod_{j=1}^d \left(1- \frac{M^2}{a}\left|(C_{\mathcal{P}_{2}}(\bold{x}))_{left}^{(j)} + \frac{a}{M^2} - x^{(j)}\right|\right)_+
\end{equation}
is a linear tensorproduct B-spline
which takes its maximum value at the center of $C_{\P_{2}}(\bold{x})$, which
is nonzero in the inner part of $C_{\P_{2}}(\bold{x})$ and which
vanishes
outside of $C_{\P_{2}}(\bold{x})$. 
\item[4.] Apply those networks to $2^d$ slightly shifted partitions of $\P_2$ to approximate $f(\bold{x})$ in supremum norm.
\end{enumerate}
\subsection{Key step 1: A recursive definition of $T_{f,q,(C_{\P_2}(\bold{x}))_{left}}(\bold{x})$}
In the proof of Theorem 1 we will use that we can compute 
the piecewise Taylor polynomial in \eqref{taylor} recursively as follows. Let $i \in \{1, \dots, M^d\}$ and $C_{\P_1}(\x)=C_{i,1}$. The recursion follows \textit{three} steps. In a first step we compute the value of $(C_{\P_1}(\x))_{left} = (C_{i,1})_{left}$ and the values of $(\partial^{\bm{\ell}} f)((\tilde{C}_{s,i})_{left})$ for $s \in \{1, \dots, M^d\}$ and $\bm{\ell} \in \N_0^d$ with $\|\bm{\ell}\|_1\leq q$. This can be done by computing the indicator function $\mathds{1}_{C_{k,1}}$ multiplied by $(C_{k,1})_{left}$ or $(\partial^{\bm{\ell}} f)((\tilde{C}_{s,k})_{left})$ for $k \in \{1, \dots, M^d\}$, respectively. Furthermore we need the value of the input $\x$ in the further recursive definition, such that we shift this value by applying the identity function. We set
\begin{align*}
\bm{\phi}_{1,1} = \left(\phi_{1,1}^{(1)}, \dots, \phi_{1,1}^{(d)}\right)= \x,
\end{align*}
\begin{align*}
\bm{\phi}_{2,1} = \left(\phi_{2,1}^{(1)}, \dots, \phi_{2,1}^{(d)}\right) &= \sum_{k \in \{1, \dots, M^d\}} (C_{k,1})_{left} \mathds{1}_{C_{k,1}}(\x)
\end{align*}
and
\begin{align*}
\phi_{3, 1}^{(\bm{\ell, s})} &= \sum_{k \in \{1, \dots, M\}^d} (\partial^{\bm{\bll}} f)\left((\tilde{C}_{s,k})_{left}\right)\mathds{1}_{C_{k,1}}(\x)
\end{align*}
for $s \in \{1, \dots, M^d\}$ and $\bm{\ell} \in \N_0^d$ with $\|\bll\|_1 \leq q$.
\\
\\ 
Let $i,j \in \{1, \dots, M^d\}$ and $(C_{\P_2}(\x))_{left} = (\tilde{C}_{j,i})_{left}$. In a second step of the recursion we compute the value of $(C_{\P_2}(\x))_{left} = (\tilde{C}_{j,i})_{left}$ and the values of $(\partial^{\bm{\ell}} f)((C_{\P_2}(\x))_{left})$ for $\bm{\ell} \in \N_0^d$ with $\|\bm{\ell}\|_1 \leq q$. It is easy to see that each cube $\tilde{C}_{s,i}$ can be defined by 
\begin{align}
\label{Aj}
\mathcal{A}^{(s)} = &\bigg\{\x \in \Rd: -x^{(k)} + \phi_{2,1}^{(k)} + v_s^{(k)} \leq 0 \ \mbox{and} \ x^{(k)} - \phi_{2,1}^{(k)} - v_s^{(k)} - \frac{2a}{M^2} < 0 \notag \\
& \hspace*{8.5cm} \mbox{for all} \ k \in \{1, \dots, d\}\bigg\}.
\end{align}
for $s \in \{1, \dots, M^d\}$. Thus in our recursion we compute for each $s \in \{1, \dots, M^d\}$ the indicator function $\mathds{1}_{\mathcal{A}^{(s)}}$ multiplied by $\bm{\phi}_{2,1} + \bm{v}_s$ or $\phi_{3,1}^{(\bm{\ell}, s)}$ for $\bm{\ell} \in \N_0^d$ with $\|\bm{\ell}\|_1 \leq q$. Again we shift the value of $\x$ by applying the indicator function. We set
\begin{align*}
\bm{\phi}_{1, 2} = \bm{\phi}_{1,1},
\end{align*}
\begin{align*}
\bm{\phi}_{2,2} = \sum_{s=1}^{M^d} (\bm{\phi}_{2,1}+\bold{v}_s) \mathds{1}_{\mathcal{A}^{(s)}} \left(\bm{\phi}_{1,1}\right)
\end{align*}
and
\begin{align*}
\phi_{3, 2}^{(\bll)} = \sum_{s=1}^{M^d} \phi_{3,1}^{(\bll, s)}\mathds{1}_{\mathcal{A}^{(s)}} \left(\bm{\phi}_{1,1}\right)
\end{align*}
for $\bll \in \N_0^d$ with $\|\bll\|_1 \leq q$. In a last step we compute the Taylor polynomial by 
\begin{align*}
\phi_{1,3} = &\sum_{\bm{\ell} \in \N_0^d: \|\bm{\ell}\|_1 \leq q} \phi_{3, 2}^{(\bm{\ell})}\frac{\left(\bm{\phi}_{1,2} - \bm{\phi}_{2,2}\right)^{\bm{\ell}}}{\bm{\ell}!}.
\end{align*}
Our next lemma shows that this recursion computes our piecewise Taylor polynomial.


\begin{lemma}
\label{le6}
  Let $p=q+s$ for some $q \in \N_0$ and $s \in (0,1]$, let $C > 0$ and $\x \in [-a,a)^d$. Let $f: \Rd \to \R$ be a $(p,C)$-smooth function and let $T_{f,q,(C_{\mathcal{P}_2}(\x))_{left}}$ be the Taylor polynomial of total degree $q$ around $(C_{\mathcal{P}_2}(\x))_{left}$. Define $\phi_{1,3}$ recursively as above. Then we have
  \[
\phi_{1,3}=T_{f,q,(C_{\mathcal{P}_2}(\x))_{left}}(\x).
  \]
\end{lemma}
\begin{proof}[\rm{\textbf{Proof of \autoref{le6}}}]
See Lemma 2 in the Supplement of \cite{KL20}.
\end{proof}

\subsection{Key step 2: Approximating $\phi_{1,3}$ by neural networks}
In this step we define a composed neural network,
which approximately computes the functions in the definition
of $\bm{\phi}_{1,1}$, $\bm{\phi}_{2,1}$, $\phi_{3,1}^{(\bll, s)}$ ($s \in \{1, \dots, M^d\}$), 
 $\bm{\phi}_{1,2}$, $\bm{\phi}_{2,2}$, $\phi_{3,2}^{(\bll)}$, 
$\phi_{1,3}$ ($\bll \in \N_0^d$ with $\|\bll\|_1 \leq q$). We will show that this neural network
is a good approximation for $f(\x)$ in case that $\x$ does 
not lie close to the boundary of any cube of $\P_2$. In particular, we show that the network 
approximates $f(\x)$ for all $\x \in \bigcup_{k \in \{1, \dots, M^{2d}\}} (C_{k,2})_{1/M^{2p+2}}^0$.


\begin{lemma}
\label{le7}
Let $\sigma:\R \to [0,1]$ be the sigmoid activation function $\sigma(x) = 1/(1+\exp(-x))$. Let $\mathcal{P}_2$ be defined as in \eqref{partition}. Let $p = q+s$ for some $q \in \N_0$ and $s \in (0,1]$, and let $C >0$. Let $f: \Rd \to \R$ be a $(p,C)$-smooth function.
    Let $1 \leq a < \infty$. Then there exists for $M \in \N$ sufficiently large (independent of the size of $a$, but 
    \begin{align*}
    M^{2p} \geq c_{2}\left(\max\left\{a, \|f\|_{C^{q}([-a,a]^d)}\right\}\right)^{5q+3} \ \mbox{and} \ M^{2p} \geq c_{11} (2ad)^{2p} C
    \end{align*}
     must hold), a neural network
$f_{net, \P_2} \in \mathcal{F}(L,r, \alpha_{net, \P_2})$ with
\begin{itemize}
\item[(i)] $L= 5+ \lceil \log_2(q+1) \rceil $
\item[(ii)] $r= \max\left\{\left(\binom{d+q}{d} + d\right)M^d(2+2d)+d, 4(q+1)\binom{d+q}{d}\right\}$
\item[(iii)] $\alpha_{net, \P_2} = c_{12}\left(\max\left\{a, \|f\|_{C^q([-a,a]^d)}\right\}\right)^{6} M^{6p+2d+6}$
\end{itemize}
such that 
\begin{align*}
&|f_{net, \mathcal{P}_2}(\x) - f(\x)| \leq  \frac{c_{2} \left(\max\left\{a, \|f\|_{C^q([-a,a]^d)} \right\}\right)^{5q+3}}{M^{2p}}
\end{align*}
holds for all $\x \in \bigcup_{k \in \{1, \dots, M^{2d}\}} \left(C_{k,2}\right)_{1/M^{2p+2}}^0$. The network value is bounded by 
\begin{align*}
\|f_{net, \mathcal{P}_2}\|_{\infty, [-a,a]^d} \leq 1+ 2^{2(d+1)}e^{2^{2d+3}ad} \max\{\|f\|_{C^q([-a,a]^d)}, 1\}.
\end{align*}
\end{lemma}
\begin{proof}[\rm{\textbf{Proof of \autoref{le7}}}]
The proof is divided into \textit{three} steps. \\
\textit{Step 1: Network architecture:} The recursively defined function $\phi_{1,3}$ of \autoref{le6} can be approximated by neural networks. In the construction we will use the network $f_{id} \in \mathcal{F}(1,1,c_{13} B_{M, id})$
 from \autoref{le1} satisfying 
 \begin{align}
\label{le7eq1}
\left|f_{id}^t(\x) - \x\right| &\leq \sum_{k=1}^t\left|f_{id}^t(\x) - f_{id}^{t-1}(\x)\right| \leq \sum_{k=1}^t\left|f_{id}(f_{id}^{t-1}(\x)) - f_{id}^{t-1}(\x)\right| \leq \frac{t}{B_{M,id}} 
\end{align}
for $\x \in [-2a,2a]^d$, $B_{M, id} \in \N_0$ and $t \in \N$ (here we choose 
\begin{align*}
R=(t-1)B_{M,id}\Vert \sigma''\Vert_{\infty} 2a^2/ (|\sigma'(t_{\sigma,id})|)
\end{align*}
in \autoref{le1}). Furthermore we use the network $f_{ind, [\bold{a}, \bold{b})} \in \mathcal{F}(2, 2d, \alpha_{ind, [\bold{a}, \bold{b})})$
of \autoref{le4} for some $\bold{a}, \bold{b} \in \Rd$, $B_{M, \epsilon}, B_M \in \N$ with
\begin{align*}
b^{(i)} - a^{(i)} \geq \frac{2}{B_M} \ \mbox{for all} \ i \in \{1, \dots, d\}
\end{align*} and
\begin{align*}
\alpha_{ind, [\bold{a}, \bold{b})} = \max\left\{5 \ln\left(B_{M, \epsilon}-1\right), \max_{i \in \{1, \dots, d\}} |b^{(i)}| B_M\ln(3)\right\}, 
\end{align*}
which satisfies
\begin{align}
\label{le7eq2}
\left|f_{ind, [\bold{a}, \bold{b})}(\x) - \mathds{1}_{[\bold{a}, \bold{b})}(\x)\right| \leq \frac{1}{B_{M, \epsilon}}
\end{align}
for
\begin{align*}
x^{(i)} \notin \Big[a^{(i)}-\frac{1}{B_M}, a^{(i)} + \frac{1}{B_{M}}\Big] \cup \Big[b^{(i)} - \frac{1}{B_{M}}, b^{(i)}+\frac{1}{B_M} \Big] \ \mbox{for all} \ i \in \{1, \dots, d\}
\end{align*}
and the network $f_{test} \in \mathcal{F}\left(3, 2d+2, \alpha_{test} \right)$
with
\begin{align*}
\alpha_{test} = c_{14} \max\left\{\left(\max\{a, \|f\|_{C^q([-a,a]^d)}\}\right)^5B_{M, \epsilon}^2, B_M\right\}
\end{align*}
and 
\begin{align*}
|s| \leq 2\max\left\{a, \|f\|_{C^q([-a,a]^d)}\right\}, 
\end{align*}
which satisfies
\begin{align}
\label{le7eq15}
\left|f_{test}(\x, \bold{a}, \bold{b}, s) - s\mathds{1}_{[\bold{a}, \bold{b})}(\x)\right| \leq \frac{1}{B_{M, \epsilon}}
\end{align}
for
\begin{align*}
x^{(i)} \notin \Big[a^{(i)}-\frac{1}{B_M}, a^{(i)} + \frac{1}{B_{M}}\Big] \cup \Big[b^{(i)} - \frac{1}{B_{M}}, b^{(i)}+\frac{1}{B_M} \Big] \ \mbox{for all} \ i \in \{1, \dots, d\}.
\end{align*}
Here we treat $1/B_M$ as $\delta$, $1/B_{M, \epsilon}$ as $\epsilon$ and  $2\max\{a, \|f\|_{C^q([-a,a]^d)}\}$ as $R$ in \autoref{le4}. For some vector $\bold{v} =(v^{(1)}, \dots, v^{(d)}) \in \R^d$ it follows
\begin{align*}
&\bold{v}f_{ind, [\bold{a}, \bold{b})}(\x) = \left(v^{(1)}f_{ind,[\bold{a}, \bold{b})}(\x), \dots, v^{(d)}f_{ind, [\bold{a}, \bold{b})}(\x)\right).
\end{align*} To compute the final Taylor polynomial we use the network 
\begin{align*}
f_{p} \in \mathcal{F}\left(\lceil \log_2(q+1) \rceil, 4(q+1) \binom{d+q}{d}, \alpha_p\right)
\end{align*}
with
\begin{align*}
\alpha_p = c_{15}\max\{\bar{r}(p),1\} B_{M,p}^2
\end{align*}
from \autoref{le3}, satisfying
\begin{align}
\label{fpeq}
&\left|f_{p}\left(\bold{z}, y_1, \dots, y_{\binom{d+q}{q}}\right) - p\left(\bold{z}, y_1, \dots, y_{\binom{d+q}{q}}\right)\right| \leq c_{18} \bar{r}(p) \frac{\left(\max\left\{2a, \|f\|_{C^q([-a,a]^d)}\right\}\right)^{5q+3}}{B_{M,p}}
\end{align}
for all $z^{(1)}, \dots, z^{(d)}, y_1, \dots, y_{\binom{d+q}{d}}$ contained in
\begin{align*}
\left[-2^{2(d+1)} \max\left\{2a, \|f\|_{C^q([-a,a]^d)}\right\}, 2^{2(d+1)}\max\left\{2a, \|f\|_{C^q([-a,a]^d)}\right\}\right],
\end{align*}
where $B_{M,p} \in \N$ satisfies
\begin{align*}
B_{M,p} =M^{2p} \geq c_{19}
 \left( \max\left\{2a, \|f\|_{C^q([-a,a]^d)}\right\}\right)^{3(q+1)}.
\end{align*}
 Here we treat $B_{M,p}$ as $R$ in \autoref{le3}.\\
\\
In the following each function of the recursion of $\phi_{1,3}$ is computed by a neural network: To compute the values of $\bm{\phi}_{1,1}$, $\bm{\phi}_{2,1}$ and $\phi_{3, 1}^{(\bll, s)}$ we use for $s \in \{1, \dots, M^d\}$ and $\bll \in \N_0^d$ with $\|\bll\|_1 \leq q$ the networks
\begin{align*}
\bm{\hat{\phi}}_{1,1}=(\hat{\phi}_{1,1}^{(1)}, \dots, \hat{\phi}_{1,1}^{(d)}) = f_{id}^2 (\x), \quad \bm{\hat{\phi}}_{2,1} = (\hat{\phi}_{2,1}^{(1)}, \dots, \hat{\phi}_{2,1}^{(d)})= \sum_{k \in \{1, \dots, M^d\}} (C_{k,1})_{left}  f_{ind, {C_{k,1}}}(\x)
\end{align*}
and 
\[
 \hat{\phi}_{3, 1}^{(\bll, s)} = \sum_{k \in \{1, \dots, M^d\}} (\partial^{\bll} f) \left((\tilde{C}_{s,k})_{left}\right) f_{ind,{C_{k,1}}}(\x).
\]
To compute $\bm{\phi}_{1,2}$, $\bm{\phi}_{2,2}$ and $\phi_{3,2}^{(\bll)}$
we use the networks
\[
\bm{\hat{\phi}}_{1,2}= (\hat{\phi}_{1,2}^{(1)},\dots, \hat{\phi}_{1,2}^{(d)})= f_{id}^{3} (\bm{\hat{\phi}}_{1,1}),
\]
\begin{align*}
&\hat{\phi}_{2,2}^{(t)}= \sum_{s=1}^{M^d} f_{test}\left(\bm{\hat{\phi}}_{1,1}, \bm{\hat{\phi}}_{2,1} + \bold{v}_s, \bm{\hat{\phi}}_{2,1} + \bold{v}_s+\frac{2a}{M^2}\mathbf{1}, \hat{\phi}_{2,1}^{(t)} + v_s^{(t)}\right) \ \mbox{and} \ \bm{\hat{\phi}}_{2,2}= (\hat{\phi}_{2,2}^{(1)}, \dots, \hat{\phi}_{2,2}^{(d)})
\end{align*}
for $t \in \{1, \dots, d\}$ and
\begin{align}
\label{neur32}
&\hat{\phi}_{3,2}^{(\bll)} = \sum_{s=1}^{M^d} f_{test}\left(\bm{\hat{\phi}}_{1,1}, \bm{\hat{\phi}}_{2,1} + \bold{v}_s, \bm{\hat{\phi}}_{2,1} + \bold{v}_s+\frac{2a}{M^2}\mathbf{1}, \hat{\phi}_{3, 1}^{(\bll, s)}\right).
\end{align}
%
Choose $\bll_1, \dots, \bll_{\binom{d+q}{d}}$ such that
\begin{align*}
\left\{\bll_1, \dots, \bll_{\binom{d+q}{d}}\right\} = \left\{(s_1, \dots, s_d) \in \N_0^d: s_1+\dots+s_d \leq q \right\} \in \N_0^d
\end{align*}
holds. 
The value of $\phi_{1,3}$ can then be computed by 
\begin{align}
\label{fp}
\hat{\phi}_{1,3}= f_p\left(\bold{z}, y_1, \dots, y_{\binom{d+q}{d}}\right),
\end{align}
where 
\begin{align*}
\bold{z}= \bm{\hat{\phi}}_{1,2}- \bm{\hat{\phi}}_{2,2} \ \mbox{and} \ y_v = \hat{\phi}_{3,2}^{(\bll_v)} 
\end{align*}
for $v \in \left\{1, \dots, \binom{d+q}{d}\right\}$. The coefficients $r_1, \dots, r_{\binom{d+q}{d}}$ in \autoref{le3} are chosen as 
\begin{align*}
r_i = \frac{1}{\ell_i^{(1)}! \cdots \ell_i^{(d)}!}, \quad i \in \left\{1, \dots, \binom{d+q}{d}\right\}.
\end{align*}

It is easy to see, that the network $\hat{\phi}_{1,3}$ forms a composed network. Here
\begin{align*}
\left(\bm{\hat{\phi}}_{1,1}, \bm{\hat{\phi}}_{2,1}, \hat{\phi}_{3,1}^{(\bll_v, s)}\right) \quad \left(v \in \left\{1, \dots, \binom{d+q}{d}\right\}, s \in \{1, \dots, M^d\}\right)
\end{align*}
is a parallelized network with depth $2$, width
\begin{align*}
r_1 = \left(\binom{d+q}{d} + d\right)M^d 2d+d 
\end{align*}
and all weights bounded by 
\begin{align*}
\alpha_1= \max\left\{\alpha_{id}, \max\{a, \|f\|_{C^q([-a,a]^d)}\}\max_{i \in \{1, \dots, M^d\}} \alpha_{ind, C_{i,1}}\right\}.
\end{align*}
Following we compute the parallelized network 
\begin{align*}
(\bm{\hat{\phi}}_{1,2}, \bm{\hat{\phi}}_{2,2}, \hat{\phi}_{3,2}^{(\bll_v)}) \quad \left(v \in \left\{1, \dots, \binom{d+q}{d}\right\}\right)
\end{align*}
with depth $5$, width
\begin{align*}
r_2 = \left(\binom{d+q}{d} + d\right) M^d (2+2d)+d 
\end{align*}
and all components bounded by $\alpha_2=\max\{\alpha_1,1\} \max\{\alpha_{test},1\}$.
%
 Thus we can conclude, that $\hat{\phi}_{1,3}$ lies in the class
\begin{align*}
\mathcal{F}\left(5+ \lceil \log_2(q+1) \rceil, r, \alpha\right)
\end{align*}
with 
\begin{align*}
r=\max\left\{r_2, 4(q+1)\binom{d+q}{d}\right\}
\end{align*}
and all weights bounded by 
\begin{align*}
\alpha = \max\{\alpha_2, \max\{\alpha_{id}, \alpha_{test},1\} \alpha_p\}.
\end{align*}
Here we have used, that
\begin{align*}
\mathcal{F}(L, r', \alpha') \subseteq \mathcal{F}(L, r, \alpha)
\end{align*}
for $r' \leq r$ and $\alpha' \leq \alpha$. Finally we set
\begin{align*}
f_{net, \mathcal{P}_2}(\x) = \hat{\phi}_{1,3}.
\end{align*}
\textit{Step 2: Approximation error of the network:} We analyze the error of the network $f_{net, \mathcal{P}_2}$ in case that 
\begin{align*}
&B_{M} = 4M^{2p+2}, \quad B_{M,id} = 12M^{2p+2}, \quad B_{M,\epsilon} = 4\max\left\{a, \|f\|_{C^q([-a,a]^d)} \right\}M^{2p+2+d}, \\
&B_{M,p} = M^{2p}
\end{align*}
and 
\begin{align}
\label{eqx}
\x \in \bigcup_{k \in \{1, \dots, M^{2d}\}} \left(C_{k,2}\right)_{1/M^{2p+2}}^0.
\end{align}
In the following we successively compute the approximation errors of each network of $\hat{\phi}_{1,3}$ to finally achieve an overall error bound for $f_{net, \P_2}$. From \eqref{le7eq1} we can conclude that 
\begin{align}
\label{le7eq11}
|\hat{\phi}_{1,1}^{(s)} -x^{(s)}| =|f_{id}^{2}(x^{(s)}) -x^{(s)}| \leq \frac{2}{12M^{2p+2}} = \frac{1}{6M^{2p+2}}
\end{align}
for all $s \in \{1, \dots, d\}$.
Since the components of $\bm{\hat{\phi}}_{1,1}$ are bounded by 
\begin{align}
\label{le7eq4}
|\hat{\phi}_{1,1}^{(s)}| \leq |\hat{\phi}_{1,1}^{(s)} -x^{(s)}| + |x^{(s)}| \leq 2a
\end{align}
for all $s \in \{1, \dots, d\}$
this leads to
\begin{align}
\label{le7eq17}
|\hat{\phi}_{1,2}^{(s)} -x^{(s)}| \leq |f_{id}^{3}(\hat{\phi}_{1,1}^{(s)}) - \hat{\phi}_{1,1}^{(s)}| + |\hat{\phi}_{1,1}^{(s)} -x^{(s)}| \leq \frac{3}{12M^{2p+2}} + \frac{1}{6M^{2p+2}} \leq \frac{1}{2M^{2p+2}}.
\end{align}
It is easy to see that inequalities \eqref{le7eq2} and \eqref{le7eq15} hold (due to the choice of $B_M$) for \eqref{eqx}. Thus we can conclude that
\begin{align}
\label{le7eq5}
\left|\hat{\phi}_{2,1}^{(s)} - \phi_{2,1}^{(s)}\right| \leq\sum_{k \in \{1, \dots, M^d\}} (C_{k,1})^{(s)}_{left} \left|f_{ind, C_{k,1}}(\x) - \mathds{1}_{C_{k,1}}(\x)\right| \leq \frac{1}{4M^{2p+2}}
\end{align}
for $s \in \{1, \dots, d\}$. Furthermore we have for $s \in \{1, \dots, M^d\}$ and $v \in \left\{1, \dots, \binom{d+q}{d}\right\}$
\begin{align*}
\left|\hat{\phi}_{3, 1}^{(\bll_v, s)} - \phi_{3, 1}^{(\bll_v, s)}\right| \leq \sum_{k \in \{1, \dots, M^d\}} \left|(\partial^{\bll_v} f)((\tilde{C}_{s,k})_{left})\right|\left|f_{ind, C_{k,1}}(\x) - \mathds{1}_{C_{k,1}}(\x)\right| \leq \frac{1}{4M^{2p+2}}.
\end{align*}
Since 
\begin{align*}
\x \in \bigcup_{k \in \{1, \dots, M^{2d}\}} \left(C_{k,2}\right)_{1/M^{2p+2}}^0
\end{align*}
by assumption, we have
\begin{align*}
|\phi_{2,1}^{(s)} - x^{(s)}| \geq \frac{1}{M^{2p+2}} \ \mbox{and} \ \left|\phi_{2,1}^{(s)} + \frac{2a}{M^2} - x^{(s)}\right| \geq \frac{1}{M^{2p+2}}
\end{align*}
for $s \in \{1, \dots, d\}$. Together with \eqref{le7eq11} and \eqref{le7eq5} we can conclude that
\begin{align}
\label{le7eq10}
|\hat{\phi}_{1,1}^{(s)}- \hat{\phi}_{2,1}^{(s)}| \geq \frac{1}{2M^{2p+2}} \ \mbox{and} \ \left|\hat{\phi}_{1,1}^{(s)} - \hat{\phi}_{2,1}^{(s)}- \frac{2a}{M}\right| \geq \frac{1}{2M^{2p+2}}.
\end{align}
In a next step we analyze the error of $\bm{\hat{\phi}}_{2,2}$ and $\hat{\phi}_{3,2}^{(\bll_v)}$. Let
\begin{align*}
\bar{\mathcal{A}}^{(s)} = &\bigg\{\x \in \R^d: -x^{(t)} +\hat{\phi}_{2,1}^{(t)} +  v_s^{(t)} \leq 0 \ \mbox{and} \ x^{(t)} - \hat{\phi}_{2,1}^{(t)} - v_s^{(t)} - \frac{2a}{M^2} < 0\\
& \hspace*{8.5cm} \mbox{for all} \ t \in \{1, \dots, d\}\bigg\}.
\end{align*}
Then we have for $t \in \{1, \dots, d\}$
\begin{align}
\left|\hat{\phi}_{2,2}^{(t)} - \phi_{2,2}^{(t)}\right| 
&\leq \sum_{s=1}^{M^d} \left|f_{test}(\bm{\hat{\phi}}_{1,1}, \bm{\hat{\phi}}_{2,1} + \bold{v}_s, \bm{\hat{\phi}}_{2,1} + \bold{v}_s + \frac{2a}{M} \mathbf{1}, \hat{\phi}_{2,1}^{(t)}+ v_s^{(t)})\right. \notag\\
& \hspace*{7cm} \left. - (\phi_{2,1}^{(t)}+ v_s^{(t)}) \mathds{1}_{\mathcal{A}^{(s)}} (\bm{\phi}_{1,1})\right| \notag\\
& \leq \sum_{s=1}^{M^d} \left|f_{test}(\bm{\hat{\phi}}_{1,1}, \bm{\hat{\phi}}_{2,1} + \bold{v}_s, \bm{\hat{\phi}}_{2,1} + \bold{v}_s + \frac{2a}{M} \mathbf{1}, \hat{\phi}_{2,1}^{(t)} + v_s^{(t)})\right. \notag\\
& \hspace*{7cm} \left. -  (\hat{\phi}_{2,1}^{(t)} + v_s^{(t)}) \mathds{1}_{\bar{\mathcal{A}}^{(s)}}(\bm{\hat{\phi}}_{1,1})\right|
\label{le7eq7}
\\
&  \quad + \sum_{s=1}^{M^d} \left|(\hat{\phi}_{2,1}^{(t)} + v_s^{(t)}) \mathds{1}_{\bar{\mathcal{A}}^{(s)}}(\bm{\hat{\phi}}_{1,1})- (\phi_{2,1}^{(t)} + v_s^{(t)})\mathds{1}_{\bar{\mathcal{A}}^{(s)}}(\bm{\hat{\phi}}_{1,1})\right|
\label{le7eq8}
\\
& \quad + \sum_{s=1}^{M^d} \left|(\phi_{2,1}^{(t)} + v_s^{(t)}) \mathds{1}_{\bar{\mathcal{A}}^{(s)}}(\bm{\hat{\phi}}_{1,1}) - (\phi_{2,1}^{(t)} + v_s^{(t)}) \mathds{1}_{\mathcal{A}^{(s)}}(\bm{\phi}_{1,1})\right|.
\label{le7eq9}
\end{align}
By \eqref{le7eq5} we can conclude that
\begin{align*}
|\hat{\phi}_{2,1}^{(t)}+v_s^{(t)}| \leq \frac{1}{4M^{2p+2}}+|\phi_{2,1}^{(t)}+v_s^{(t)}| \leq \frac{1}{4M^{2p+2}} + a \leq 2a.
\end{align*}
By \eqref{le7eq10} it follows that $\hat{\phi}_{1,1}^{(t)}$ is not contained in 
\begin{align*}
 &\left(\hat{\phi}_{2,1}^{(t)}+v_s^{(t)}-\frac{1}{B_M}, \hat{\phi}_{2,1}^{(t)} + v_s^{(t)} + \frac{1}{B_M}\right) \\ 
 &\cup  \left(\hat{\phi}_{2,1}^{(t)} + v_s^{(t)} + \frac{2a}{M} - \frac{1}{B_M}, \hat{\phi}_{2,1}^{(t)} + v_s^{(t)} + \frac{2a}{M}+\frac{1}{B_M}\right)
\end{align*}
for all $t \in \{1, \dots, d\}$. Thus we can conclude by \eqref{le7eq15} and with the choice of $B_{M, \epsilon}$ that inequality \eqref{le7eq7} is bounded by 
\begin{align*}
&\sum_{s=1}^{M^d} \left|f_{test}(\bm{\hat{\phi}}_{1,1}, \bm{\hat{\phi}}_{2,1}+ \bold{v}_s, \bm{\hat{\phi}}_{2,1} + \bold{v}_s + \frac{2a}{M}\mathbf{1}, \hat{\phi}_{2,1}^{(t)} + v_s^{(t)}) -  (\hat{\phi}_{2,1}^{(t)} + v_s^{(t)}) \mathds{1}_{\bar{\mathcal{A}}^{(s)}}(\bm{\hat{\phi}}_{1,1})\right|\leq \frac{1}{4M^{2p+2}}.
\end{align*}
Using \eqref{le7eq5} we get a bound for inequality \eqref{le7eq8}:
\begin{align*}
\sum_{s=1}^{M^d} \left|(\hat{\phi}_{2,1}^{(t)}+ v_s^{(t)}) \mathds{1}_{\bar{\mathcal{A}}^{(s)}}(\bold{\bm{\hat{\phi}}}_{1,1}) - (\phi_{2,1}^{(t)} + v_s^{(t)}) \mathds{1}_{\bar{\mathcal{A}}^{(s)}}(\bm{\hat{\phi}}_{1,1})\right| \leq \frac{1}{4M^{2p+2}}.
\end{align*}
Replacing $\mathcal{A}^{(s)}$ (see \eqref{Aj}) by $\bar{\mathcal{A}}^{(s)}$ means replacing $\bm{\phi}_{2,1}$  by $\bm{\hat{\phi}}_{2,1}$. This leads to a change of the boundaries of the cubes of at most  $1/4M^{2p+2}$. Furthermore the computation of $\bm{\phi}_{1,1}$ by $\bm{\hat{\phi}}_{1,1}$ induces an error of at most $1/4M^{2p+2}$.  Combining this with the fact that 
\begin{align*}
\x \in \bigcup_{k \in \{1, \dots, M^{2d}\}} \left(C_{k,2}\right)_{1/M^{2p+2}}^0
\end{align*}
leads to
\begin{align*}
 \mathds{1}_{\bar{\mathcal{A}}^{(s)}}(\bm{\hat{\phi}}_{1,1}) =  \mathds{1}_{\mathcal{A}^{(s)}}(\x), 
\end{align*}
for each $s \in \{1, \dots, M^d\}$, such that \eqref{le7eq9} is zero. Summarizing the above results we can conclude that
\begin{align}
\label{le7eq16}
\left|\hat{\phi}_{2,2}^{(t)} -\phi_{2,2}^{(t)}\right| \leq \frac{1}{4M^{2p+2}} + \frac{1}{4M^{2p+2}} + 0 = \frac{1}{2M^{2p+2}}.
\end{align}
With the same argumentation as above we can conclude that
\begin{align*}
\left|\hat{\phi}_{3, 2}^{(\bll_v)} - \phi_{3, 2}^{(\bll_v)}\right| \leq \frac{1}{2M^{2p+2}}
\end{align*}
for $v \in \left\{1, \dots, \binom{d+q}{d}\right\}$. Then it follows that 
\begin{align}
\label{fs1}
\left|\hat{\phi}_{1,2}^{(t)} - \hat{\phi}_{2,2}^{(t)}\right| &\leq \left|\hat{\phi}_{1,2}^{(t)} -x^{(t)}\right| + \left|x^{(t)}-\phi_{2,2}^{(t)}\right| + \left|\phi_{2,2}^{(t)} - \hat{\phi}_{2,2}^{(t)}\right| \leq \frac{1}{2M^{2p+2}} + 2a + \frac{1}{2M^{2p+2}} \leq 3a.
\end{align}
and 
\begin{align*}
\left|\hat{\phi}_{3, 2}^{(\bll_v)}\right| \leq \left|\phi_{3, 2}^{(\bll_v)}-\hat{\phi}_{3, 2}^{(\bll_v)}\right| + \left|\phi_{3, 2}^{(\bll_v)}\right| \leq \frac{1}{2M^{2p+2}} + \|f\|_{C^q([-a,a]^d)} \leq 2\|f\|_{C^q([-a,a]^d)}.
\end{align*}
Therefore the input of $f_p$ in \eqref{fp} is contained in the interval where \eqref{fpeq} holds. Since $B_{M,p} =M^{2p}$
we get
\begin{align*}
  &
  \left|f_{net, \mathcal{P}_2}(\x) - T_{f,q,(C_{\mathcal{P}_2}(\x))_{left}}(\x)\right|
  =
  \left|\hat{\phi}_{1,3} - \phi_{1,3} \right| \leq \frac{c_{2} \left(\max\left\{a, \|f\|_{C^q([-a,a]^d)}\right\}\right)^{5q+3} }{M^{2p}}.
\end{align*}
This together with \autoref{le5} implies the first assertion of the lemma. The value of the network is then bounded by
\begin{align*}
\left|f_{net, \mathcal{P}_2}(\x)\right| \leq &\left|f_{net, \mathcal{P}_2}(\x) - T_{f,q,(C_{\mathcal{P}_2}(\x))_{left}}(\x)\right| + \left|T_{f,q,(C_{\mathcal{P}_2}(\x))_{left}}(\x) - f(\x) \right|+ \left|f(\x)\right| \\
\leq & 3 \max\left\{\|f\|_{\infty, [-a,a]^d},1\right\},
\end{align*}
where we have used that
\begin{align*}
M^{2p} \geq c_{2}\left(\max\left\{a, \|f\|_{C^q([-a,a]^d)}\right\}\right)^{5q+3}
\end{align*}
and $M^{2p} \geq c_{11} (2ad)^{2p} C$.
\\
\\
\textit{Step 3: Bound of the network's value:} We analyze the bound of $f_{net, \mathcal{P}_2}$ in case that
\begin{align*}
\x \in \bigcup_{k \in \{1, \dots, M^{2d}\}} C_{k,2} \textbackslash (C_{k,2})_{1/M^{2p+2}}^0.
\end{align*}
In particular we assume that $\x \in \bigcap_{k \in I} C_{k,2} \textbackslash (C_{k,2})_{1/M^{2p+2}}^0$, where $I \subset \{1, \dots, M^{2d}\}$ and $|I| \leq 2^d$. Note that for at most $2^d$ cubes of the partition $\mathcal{P}_1$ or $\mathcal{P}_2$ $\x$ can be close to the boundaries. For those cubes, the networks $f_{test}$ and $f_{ind, C_{k,1}}$ $(k \in \{1, \dots, M^d\})$ are not a good approximation (see \autoref{le4}). Particulary, this implies
\begin{align*}
\left|\hat{\phi}_{2,1}^{(t)}\right| \leq \sum_{k \in \{1, \dots, M^d\}} \left|(C_{k,1})_{left}^{(t)} f_{ind, C_{k,1}}(\x)\right| \leq a(M^d-2^d) \frac{1}{B_{M, \epsilon}} +2^da \leq 2^{d+1}a
\end{align*}
for $t \in \{1, \dots, d\}$, where we have used that for $k \in \{1, \dots, M^{d}\}\textbackslash I$ 
\begin{align*}
\mathds{1}_{C_{k,1}}(\x) =0 \ \mbox{and} \ |f_{ind, C_{k,1}}(\x)| \leq \frac{1}{B_{M, \epsilon}}
\end{align*}
and that $f_{ind, C_{k,1}}(\x) \in [0,1]$ for $\x \in \Rd$.
With the same argumentation we can bound
\begin{align*}
\left|\hat{\phi}_{3, 1}^{(\bll, s)}\right| \leq 2^{d+1}\max\{\|f\|_{C^q([-a,a]^d)},1\}
\end{align*}
for $s \in \{1, \dots, M^d\}$ and $\bll \in \N_0^d$ with $\|\bll\|_1\leq q$.
Since $f_{test}$ also produces for at most $2^d$ summands in \eqref{neur32} not a good approximation (and is of size $c_{20}/B_{M, \epsilon}$ in all other cases), this leads to 
\begin{align*}
\left|\hat{\phi}_{3, 2}^{(\bll)}\right| \leq 2^{2(d+1)} \max\{\|f\|_{C^q([-a,a]^d)}, 1\} \ \mbox{and} \ \left|\hat{\phi}_{2,2}^{(t)}\right| \leq 2^{2(d+1)} a
\end{align*}
for $t \in \{1, \dots, d\}$.
We finally conclude that
\begin{eqnarray*}
\left|f_{net, \mathcal{P}_2}(\x)\right| &\leq &\left|f_p\left(\bold{z},y_1, \dots, y_{\binom{d+q}{d}}\right) - p\left(\bold{z},y_1, \dots, y_{\binom{d+q}{d}}\right)\right|+ \left|p\left(\bold{z},y_1, \dots, y_{\binom{d+q}{d}}\right)\right|\\
&\leq & 1 + \sum_{\bll \in \N_0^d: \|\bll\|_1 \leq q} \frac{2^{2(d+1)}}{\bll!}\max\{\|f\|_{C^q([-a,a]^d)},1\} \left(2^{2(d+1)+1}a\right)^{\|\bll\|_1}\\
&
\leq&
1 +
2^{2(d+1)}\max\{\|f\|_{C^q([-a,a]^d)}, 1\}
\left(
\sum_{k=0}^\infty \frac{(2^{2(d+1)+1}a)^{k}}{k!}
\right)^d
\\
&
=&  1+ 2^{2(d+1)}e^{2^{2(d+1)+1}ad} \max\{\|f\|_{C^q([-a,a]^d)}, 1\}.
\end{eqnarray*}
\end{proof}

%

\subsection{Key step 3: A network for the approximation of $w_{\mathcal{P}_2}(\x) f(\x)$}
In order to approximate $f(\x)$ in supremum norm, we will further use the
neural network of \autoref{le7} to construct a network which approximates $w_{\P_2}(\x)f(\x)$,
where
\begin{equation}
  \label{w_vb}
w_{\P_2}(\x) = \prod_{j=1}^d \left(1- \frac{M^2}{a}\left|(C_{\mathcal{P}_{2}}(x))_{left}^{(j)} + \frac{a}{M^2} - x^{(j)}\right|\right)_+
\end{equation}
is a linear tensorproduct B-spline
which takes its maximum value at the center of $C_{\P_{2}}(\x)$, which
is nonzero in the inner part of $C_{\P_{2}}(\x)$ and which
vanishes
outside of $C_{\P_{2}}(\x)$.
It is easy to see that
$w_{\P_2}(\x)$
is less than or equal to $1/M^{2p}$ in case that $\x$ is contained in
\[
\bigcup_{k \in \{1, \dots, M^{2d}\}}
C_{k, 2} \setminus
(C_{k, 2})_{1/M^{2p+2}}^0
.
\]
Since $w_{\P_2}(\x)$ is close to zero close to the boundary
of $C_{\P_2}(\x)$ it will be possible to construct this neural network
such that it approximates $w_{\P_2}(\x)f(\x)$ in supremum norm.
\\
\\
To construct this DNN, we need to approximate $w_{\P_2}(\x)$ in a first step. According to \eqref{w_vb} the function $w_{\P_2}(\x)$ forms a product of functions $g(z)=z_+=\max\{z,0\}$ $(z \in \R)$. The following lemma helps us to approximate $g(z)$ by a DNN.
%
%
%

\begin{lemma}\label{le8}
  Let $\sigma: \mathbb{R} \to [0,1]$ be the sigmoid activation function $\sigma(x) = 1/(1+\exp(-x))$. Let $f_{mult}$ be the neural network of \autoref{le2}
    and let $f_{id}$ be the network of \autoref{le1}.
Assume
\begin{equation}
\label{le8eq1}
a \geq 1 \quad \mbox{and} \quad
R
\geq
\max \left(
2\| \sigma^{\prime \prime}\|_{\infty} a
, 1
\right).
\end{equation}
Then the neural network $f_{ReLU} \in \F(2,4,c_{17}R^2)$ 
satisfies
\begin{equation*}
|f_{ReLU}(x) - \max\{x,0\}| \leq208
\max \left\{
\| \sigma^{\prime \prime}\|_{\infty},
\| \sigma^{\prime \prime \prime}\|_{\infty},1
\right\}\frac{a^3}{R}
\end{equation*}
for all $x \in [-a, a]$. 
\end{lemma}

\begin{proof}[\rm{\textbf{Proof of \autoref{le8}}}]
See Lemma 3 in \cite{KL20a}.
\end{proof}

\begin{lemma}
\label{le9}
Let $\sigma: \R \to [0,1]$ be the sigmoid activation function $\sigma(x) = 1/(1+\exp(-x))$. Let $1 \leq a < \infty$ and $M \geq c_{28} (d-1) 4^{5d-2} 2^{5d-2}$. Let $\mathcal{P}_{2}$
be the partition defined in \eqref{partition} and let $w_{\P_2}$ be the corresponding weight defined by \eqref{w_vb}. Then there exists a neural network
\begin{align*}
f_{w_{\P_2}} \in \mathcal{F}\left(7+\lceil\log_2(d)\rceil, \max\left\{12d, 2d+M^dd(2+2d)\right\}, \alpha_{w_{\P_2}}\right)
\end{align*}
with
\begin{align*}
\alpha_{w_{\P_2}} = c_{3}\left(\max\left\{a, \|f\|_{C^q([-a,a]^d)} \right\}\right)^7 M^{6p+6+2d}
\end{align*}
such that
\begin{align*}
\left|f_{w_{\P_2}}(\x) - w_{\P_2}(\x)\right| \leq  \frac{ \max\left\{c_{3}, 2^d, 12d\right\}}{M^{2p}}
\end{align*}
for $\x \in \bigcup_{k \in \{1, \dots, M^{2d}\}} (C_{k,2})_{1/M^{2p+2}}^0$ and 
\begin{align*}
|f_{w_{\P_2}}(\x)| \leq 2^{d+1}
\end{align*}
for $\x \in [-a,a)^d$.
\end{lemma}
\begin{proof}[\rm{\textbf{Proof of \autoref{le9}}}]
Using \autoref{le8} this proof follows as a straightforward modification from the proof of Lemma 10 in the Supplement of \cite{KL20}. A complete proof can be found in the Supplement. 
\end{proof}
Since $f_{w_{\P_2}}$ is only a good approximation of $w_{\P_2}$ for $\x \in \bigcup_{k \in \{1, \dots, M^{2d}\}} (C_{k,2})_{1/M^{2p+2}}^0$, we use a special construction of a DNN to approximate $w_{\P_2}(\x) f(\x)$ in supremum norm. In particular, we construct a network, which is a good approximation in case that 
\begin{align*}
\x \in \bigcup_{k \in \{1, \dots, M^{2d}\}} (C_{k,2})_{2/M^{2p+2}}^0
\end{align*}
 and which approximately vanishes for 
 \begin{align*}
 \x \in \bigcup_{k \in \{1, \dots, M^{2d}\}} C_{k,2} \setminus (C_{k,2})_{1/M^{2p+2}}^0.
 \end{align*}
Therefore we approximate the function
\begin{align}
\label{eq1000}
&\left(f_{net, \P_2}(\x) - B_{true} \mathds{1}_{
    \bigcup_{k \in \{1, \dots, M^{2d}\}}
    C_{k,2} \setminus (C_{k,2})_{1/M^{2p+2}}^0
}(\x)\right)_+ \notag\\
&- \left(-f_{net, \P_2}(\x) - B_{true} \mathds{1}_{
    \bigcup_{k \in \{1, \dots, M^{2d}\}}
    C_{k,2} \setminus (C_{k,2})_{1/M^{2p+2}}^0
}(\x)\right)_+
\end{align}
by a DNN. Here $B_{true}$ is the bound of the network $f_{net, \P_2}$ given in \autoref{le7}. Now it is easy to see: In case that $\x$ is close to the boundaries of a cube of $\P_2$ the value of the indicator function gets $1$, which in turn means that \eqref{eq1000} gets zero. Otherwise, in case that $\x$ lies in the inner part of one of the cubes of $\P_2$,  the indicator function is $0$ and  the value of $f(\x)$ is approximated by $f_{net, \P_2}$ as described in \autoref{le7}. 
\\
\\
The next lemma helps us to approximate $\mathds{1}_{
    \bigcup_{k \in \{1, \dots, M^{2d}\}}
    C_{k,2} \setminus (C_{k,2})_{1/M^{2p+2}}^0
}(\x)$ by a DNN. This network is approximately $1$ for $\x \in \bigcup_{k \in \{1, \dots, M^{2d}\}} C_{k,2} \setminus \left(C_{k,2}\right)_{1/M^{2p+2}}^0$ and approximately $0$ for $\x \in \bigcup_{k \in \{1, \dots, M^{2d}\}} (C_{k,2})_{2/M^{2p+2}}^0$.
\begin{lemma}
\label{le10}
Let $\sigma: \R \to [0,1]$ be the sigmoid activation function $\sigma(x) =1/(1+\exp(-x))$. Let $1 \leq a < \infty$ and $\epsilon \in (0,1)$. Let
$\mathcal{P}_{1}$ and $\mathcal{P}_{2}$
be the partitions defined in \eqref{partition} and let $M \in \N$. Then there exists a neural network 
\begin{align*}
f_{check, \mathcal{P}_{2}} \in \mathcal{F}\left(6, (2d+2)dM^d+d, \alpha_{check, \P_2}\right)
\end{align*}
with
\begin{align*}
\alpha_{check, \P_2} = c_{32} \max\left\{a, \|f\|_{C^q([-a,a]^d)} \right\}^2 M^{4(p+1+d)}
\end{align*}
satisfying
\begin{align*}
  \left|f_{check, \mathcal{P}_{2}}(\x) -\mathds{1}_{
    \bigcup_{k \in \{1, \dots, M^{2d}\}}
    C_{k,2} \setminus (C_{k,2})_{1/M^{2p+2}}^0
}(\x)\right| \leq \frac{1}{M^{2p+2}}
\end{align*}
for $\x \notin \bigcup_{k \in \{1, \dots, M^{2d}\}} (C_{k,2})_{1/M^{2p+2}}^0 \textbackslash (C_{k,2})_{2/M^{2p+2}}^0$ and 
\begin{align*}
f_{check, \mathcal{P}_{2}}(\x) \in [0,1]
\end{align*}
for $\x \in [-a,a)^d$. 
\end{lemma}


\begin{proof}[\rm{\textbf{Proof of \autoref{le10}}}] Throughout the proof we assume that $i \in \{1, \dots, M^d\}$ satisfies $C_{\P_1}(\x)=C_{i,1}$. The proof is divided into \textit{two} steps.\\
\textit{Step 1: Network architecture:} We again use a two scale approximation. In a first step we compute the position of $(C_{\P_1}(\x))_{left}$. In a second step, it is then enough to approximate the function 
\begin{align*}
\mathds{1}_{\bigcup_{j \in \{1, \dots, M^{d}\}} \tilde{C}_{j,i} \setminus (\tilde{C}_{j,i})_{1/M^{2p+2}}^0}(\x), 
\end{align*}
which is in turn equal to
\begin{align*}
\mathds{1}_{
    \bigcup_{k \in \{1, \dots, M^{2d}\}}
    C_{k,2} \setminus (C_{k,2})_{1/M^{2p+2}}^0
}(\x).
\end{align*}
The value of $(C_{\P_1}(\x))_{left}$ is computed by a network $\bm{\hat{\phi}}_{2,1}$ defined as in in the proof of \autoref{le7}. Since $\bm{\hat{\phi}}_{2,1}$ is only a good approximation in case that
\begin{align*}
\x \in \bigcup_{k \in \{1, \dots, M^d\}} (C_{k,1})_{1/M^{2p+2}}^0
\end{align*}
we further need to check by 
\begin{align}
\label{le10eq1}
f_1(\x)=\mathds{1}_{\bigcup_{k \in \{1, \dots, M^{d}\}} C_{k,1} \setminus (C_{k,1})_{1/M^{2p+2}}^0}(\x) =1-\sum_{k \in \{1, \dots, M^d\}} \mathds{1}_{(C_{k,1})_{1/M^{2p+2}}^0}(\x)
\end{align}
if $\x$ is close to the boundaries of the coarse grid $\P_1$. 
This can be done with the networks $f_{ind,  (C_{k,1})_{5/4M^{2p+2}}^0}$ of \autoref{le4}, satisfying
\begin{align*}
\left|f_{ind, (C_{k,1})_{5/4M^{2p+2}}^0}(\x) - \mathds{1}_{(C_{k,1})_{1/M^{2p+2}}^0}(\x)\right| \leq \frac{1}{B_{M, \epsilon}}
\end{align*}
for 
\begin{align*}
\x \notin & \left[((C_{k,1})_{5/4M^{2p+2}}^0)_{left} - \frac{1}{4M^{2p+2}} \bold{1}, ((C_{k,1})_{5/4M^{2p+2}}^0)_{left} + \frac{1}{4M^{2p+2}} \bold{1}\right]\\ &\cup \left[((C_{k,1})_{5/4M^{2p+2}}^0)_{left} +\frac{2a}{M}\bold{1} - \frac{1}{4M^{2p+2}} \bold{1}, ((C_{k,1})_{5/4M^{2p+2}}^0)_{left} +\frac{2a}{M}\bold{1} + \frac{1}{4M^{2p+2}} \bold{1}\right], 
\end{align*}
where we choose $\delta = 1/B_M = 1/(4M^{2p+2})$ and \begin{align*}
B_{M,\epsilon} = 4\max\left\{a, \|f\|_{C^q([-a,a]^d)}\right\}M^{2p+2+2d}.
\end{align*} Thus we set
\begin{align*}
\hat{f}_1(\x) = 1-\sum_{i \in \{1, \dots, M^d\}} f_{ind, (C_{i,1})_{5/4M^{2p+2}}^0}(\x). 
\end{align*}
To shift the value of $\x$ in the next hidden layers we use the network $\bm{\hat{\phi}}_{1,1}$ of \autoref{le7}. The describes parallelized network $(\bm{\hat{\phi}}_{1,1}, \bm{\hat{\phi}}_{2,1}, \hat{f}_1)$ needs $2$ hidden layers and $(2d+2)dM^d+d$ neurons per layer ($d$ neurons for $\bm{\hat{\phi}}_{1,1}$, $2d^2M^d$ neurons for $\bm{\hat{\phi}}_{2,1}$ and $2dM^d$ neurons for $\hat{f}_1(\x)$). The weights of the network are bounded by 
\begin{align*}
\alpha_{ind}=\max\{5\ln(B_{M, \epsilon}-1), a4M^{2p+2}\ln(3)\}.
\end{align*}
 In the third, forth and fifth hidden layer the network then approximates the function
\[
f_2(\x)=\mathds{1}_{
    \bigcup_{j \in \{1, \dots, M^{d}\}}
    \tilde{C}_{j,i} \setminus (\tilde{C}_{j,i})_{1/M^{2p+2}}^0
}(\x) = 1-\sum_{j \in \{1, \dots, M^d\}} \mathds{1}_{(\tilde{C}_{j,i})_{1/M^{2p+2}}^0}(\x)
\]
by
\begin{align*}
\hat{f}_2(\x) =1-\sum_{j \in \{1, \dots, M^d\}} f_{test}\left(\bm{\hat{\phi}}_{1,1}, \bm{\hat{\phi}}_{2,1} +\bold{v}_j +\frac{5}{4M^{2p+2}}\mathbf{1}, \bm{\hat{\phi}}_{2,1} + \bold{v}_j+\frac{2a}{M^2}\mathbf{1}-\frac{5}{4M^{2p+2}} \mathbf{1}, 1\right), 
\end{align*}
where $f_{test}$ is the network of \autoref{le4} b) satisfying
\begin{align*}
&\left|f_{test}\left(\bm{\hat{\phi}}_{1,1}, \bm{\hat{\phi}}_{2,1} +\bold{v}_j +\frac{5}{4M^{2p+2}}\mathbf{1}, \bm{\hat{\phi}}_{2,1} + \bold{v}_j+\frac{2a}{M^2}\mathbf{1}-\frac{5}{4M^{2p+2}} \mathbf{1}, 1\right)\right. \\
& \left. \hspace*{4cm} - \mathds{1}_{\left[\bm{\hat{\phi}}_{2,1} +\bold{v}_j +\frac{5}{4M^{2p+2}}\mathbf{1}, \bm{\hat{\phi}}_{2,1} + \bold{v}_j+\frac{2a}{M^2}\mathbf{1}-\frac{5}{4M^{2p+2}} \mathbf{1}\right]}(\bm{\hat{\phi}}_{1,1})\right| \leq \frac{1}{B_{M, \epsilon}}
\end{align*}
for 
\begin{align*}
\bm{\hat{\phi}}_{1,1} \in &\left[\bm{\hat{\phi}}_{2,1} +\bold{v}_j +\frac{1}{M^{2p+2}}\mathbf{1}, \bm{\hat{\phi}}_{2,1} +\bold{v}_j +\frac{6}{4M^{2p+2}}\mathbf{1}\right]\\
 &\cup \left[\bm{\hat{\phi}}_{2,1} + \bold{v}_j+\frac{2a}{M^2}\mathbf{1}-\frac{1}{M^{2p+2}} \mathbf{1}, \bm{\hat{\phi}}_{2,1} + \bold{v}_j+\frac{2a}{M^2}\mathbf{1}+\frac{6}{4M^{2p+2}} \mathbf{1}\right]
\end{align*}
where we also choose $\delta = 1/B_M=1/(4M^{2p+2})$.
The final network is then of the form 
\begin{align*}
&f_{check, \mathcal{P}_{2}}(\x) = \sigma\Bigg(-B_1\Bigg(\frac{1}{2} - 
\hat{f}_2(\x) - B_2f_{id}^3\left(f_1(\bold{x})\right)\Bigg)\Bigg),
\end{align*}
with $B_1=4\ln(M^{2p+2}-1)$ and $B_2=M^d$, where $f_{id}$ is the network of \autoref{le1} satisfying
\begin{align}
\label{le10eq3}
| f^3_{id}(\x)-\x|
\leq
\frac{1}{2M^{2p+2+d}}, 
\end{align}
for $\x \in [-2,2]$ (where we choose $R \geq 6
\| \sigma^{\prime \prime}\|_{\infty}/|\sigma^\prime(t_{\sigma,id})|)M^{2p+2+d}$ in \autoref{le1}).
Since each network $f_{test}$ is of width $(2d+2)$ and $f_{id}$ needs $1$ neuron per layer, the final network $f_{check, \P_2}$
has width
\begin{align*}
r=\max\{(2d+2)dM^d+d, M^d(2d+2)+1\} = (2d+2)dM^d+d
\end{align*}
 and depth $6$. The weights of $f_{test}$ are bounded by 
\begin{align*}
\alpha_{test} = c_{10} \max\left\{B_{M, \epsilon}^2, \frac{1}{B_M}\right\} = c_{10} \max\left\{B_{M, \epsilon}^2, 4M^{2p+2}\right\}
\end{align*}
according to \autoref{le4} b). Thus the weights of the network $f_{check, \P_2}$ are bounded by
\begin{align*}
\max\{B_2\alpha_{ind}, B_1\alpha_{id}, B_1 \alpha_{test}\} &=B_1\alpha_{test} \\
&= 4\ln(M^{2p+2}-1)c_{10}\left(4\max\left\{a, \|f\|_{C^q([-a,a]^d)}\right\}M^{2p+2+d}\right)^2\\
& \leq c_{32} \max\left\{a, \|f\|_{C^q([-a,a]^d)}\right\}^2 M^{4(p+d)+5}.
\end{align*}

\textit{Step 2: Approximation error of the network:} We analyze \textit{three} different cases for the approximation error of our network. In the first case we assume that
\begin{equation}
\label{le10eq2}
\x \notin \bigcup_{k \in \{1, \dots, M^d\}} (C_{k,1})^0_{1/M^{2p+2}}.
\end{equation}
Then it follows by \autoref{le4} a) (with the choice of $B_M$) that
\begin{align*}
\hat{f}_1(\x) = 1-\sum_{i \in \{1, \dots, M^d\}} f_{ind, (C_{i,1})_{5/(4M^{2p+2})}^0}(\x) \geq 1-M^d \frac{1}{B_{M, \epsilon}} \geq 1-\frac{1}{4M^{2p+2+d}}
\end{align*}
from which we can conclude by \eqref{le10eq3} (since $\hat{f}_1(\x)\in [0,1]$)
\begin{align*}
f_{id}^3(\hat{f}_1(\x)) \geq f_{id}^3\left(1-\frac{1}{4M^{2p+2}}\right) \geq 1-\frac{1}{4M^{2p+2+d}}-\frac{1}{2M^{2p+2+d}} =1-\frac{3}{4M^{2p+2+d}}.
\end{align*}
Furthermore we have
\begin{align*}
\hat{f}_2(\x) &= 1-\sum_{j \in \{1, \dots, M^d\}} f_{test}\left(\hat{\phi}_{1,1}, \hat{\phi}_{2,1} +\bold{v}_j +\frac{5}{4M^{2p+2}}\mathbf{1},  \hat{\phi}_{2,1} + \bold{v}_j+\frac{2a}{M^2}\mathbf{1}-\frac{5}{4M^{2p+2}} \mathbf{1}, 1\right)\\
& \geq -M^d
\end{align*}
where we have used that $f_{test}$ is bounded by $1$ (since we choose $s=1$ in \autoref{le4} b)). With $B_1=4\ln(M^{2p+2}-1)$ and $B_2=M^d$ it follows 
\begin{align*}
f_{check, \P_2}(\x) &= \sigma\left(-B_1\left(\frac{1}{2} - \hat{f}_2(\x) - B_2f_{id}^3\left(\hat{f}_1(\x)\right)\right)\right)\\
&\geq \sigma\left(-B_1\left(\frac{1}{2}+M^d-B_2\left(1-\frac{3}{4M^{2p+2+d}}\right)\right)\right)\\
&\geq \sigma\left(\frac{1}{2}B_1\right) \geq 1-\frac{1}{M^{2p+2}},
\end{align*}
where we have used that for $\kappa > 0$
\begin{align*}
\sigma(x) \geq 1-\kappa \ \mbox{for} \ x \geq \ln\left(\frac{1}{\kappa}-1\right).
\end{align*}
In our second case we assume that
\begin{align}
\label{check1}
\x \in \bigcup_{k \in \{1, \dots, M^d\}} (C_{k,1})_{1/M^{2p+2}}^0
\end{align}
and 
\begin{align*}
\x \notin \bigcup_{k \in \{1, \dots, M^{2d}\}} (C_{k,2})_{1/M^{2p+2}}^0.
\end{align*}
Then we have $|\hat{\phi}_{2,1}^{(s)} - (C_{\P_1}(\x))_{left}^{(s)}| = |\hat{\phi}_{2,1}^{(s)}-(C_{i,1})_{left}^{(s)}| \leq 1/(4M^{2p+2})$ for $s \in \{1, \dots, d\}$ (see \eqref{le7eq5} in \autoref{le7})  and $|\hat{\phi}_{1,1}^{(s)}- x^{(s)}| \leq 1/6M^{2p+2}$ for $s \in \{1, \dots, d\}$ (see \eqref{le7eq11} in \autoref{le7}). 
Furthermore it follows by \autoref{le4} a)
 \begin{align*}
-M^d \leq \hat{f}_1(\x) = 1-\sum_{i \in \{1, \dots, M^d\}} f_{ind, (C_{i,1})_{5/(4M^{2p+2})}^0}(\x) \leq 1
 \end{align*}
 where we have used that $f_{ind, (C_{i,1})_{5/(4M^{2p+2})}^0} \in [0,1]$. With the error bound \eqref{le10eq3} we have
 \begin{align*}
-M^d-\frac{1}{2M^{2p+2+d}} \leq f_{id}^3\left(\hat{f}_1(\x)\right) \leq 1+\frac{1}{2M^{2p+2+d}}.
 \end{align*}
Since 
\begin{align*}
\x \notin \bigcup_{k \in \{1, \dots, M^{2d}\}} (C_{k,2})_{1/M^{2p+2}}^0.
\end{align*}
there is at least one $k \in \{1, \dots, M^{2d}\}$ satisfying
\begin{align*}
\left|x^{(s)} - (C_{k,2})^{(s)}_{left}\right| \leq \frac{1}{M^{2p+2}} \ \mbox{or} \ \left|x^{(s)} - (C_{k,2})_{left}^{(s)} - \frac{2a}{M^2}\right| \leq \frac{1}{M^{2p+2}}
\end{align*}
for some $s \in \{1, \dots, d\}$. This means that for some $k \in \{1, \dots, M^{d}\}$ we have
\begin{align*}
\left|\hat{\phi}_{1,1}^{(s)} - \hat{\phi}_{2,1}^{(s)} - v_k^{(s)}\right| &\leq \left|\hat{\phi}_{1,1}^{(s)} - x^{(s)}\right|+\left|x^{(s)} - \phi_{2,1}^{(s)} - v_k^{(s)}\right|+\left|\phi_{2,1}^{(s)} - v_k^{(s)} - \hat{\phi}_{2,1}^{(s)} + v_k^{(s)}\right|\\
& \leq \frac{1}{6M^{2p+2}} + \frac{1}{M^{2p+2}} + \frac{1}{4M^{2p+2}} \leq \frac{6}{4M^{2p+2}}
\end{align*}
or
\begin{align*}
\left|\hat{\phi}_{1,1}^{(s)} - \hat{\phi}_{2,1}^{(s)} - v_k^{(s)}-\frac{2a}{M^2}\right| \leq \frac{6}{4M^{2p+2}}.
\end{align*}
We set
\begin{align*}
&\bar{\mathcal{A}}^{(j)}= \bigg\{\x \in \Rd: -x^{(k)} + \hat{\phi}_{2,1}^{(k)} + v_j^{(k)} \leq 0 \ \mbox{and} \ x^{(k)} -\hat{\phi}_{2,1}^{(k)} - v_j^{(k)} - \frac{2a}{M^2}  \leq 0 \\
& \hspace*{10cm}  \mbox{for all} \ k \in \{1, \dots, d\}\bigg\}.
\end{align*}
With the results above we see that
\begin{align*}
\mathds{1}_{\bigcup_{k \in \{1, \dots, M^{2d}\}} C_{k,2} \textbackslash (C_{k,2})_{1/M^{2p+2}}^0}(\x) = 1
\end{align*}
implies
\begin{align*}
\mathds{1}_{\bigcup_{k \in \{1, \dots, M^{d}\}} \bar{\mathcal{A}}^{(k)} \textbackslash (\bar{\mathcal{A}}^{(k)})_{6/(4M^{2p+2})}^0}(\bm{\hat{\phi}}_{1,1})=1-\sum_{k \in \{1, \dots, M^d\}} \mathds{1}_{(\bar{\mathcal{A}}^{(k)})_{6/(4M^{2p+2})}^0}(\bm{\hat{\phi}}_{1,1})=1.
\end{align*}
The second indicator function can then be approximated by our network, since
\begin{align*}
&1-\sum_{j \in \{1, \dots, M^d\}} f_{test}\left(\bm{\hat{\phi}}_{1,1}, \bm{\hat{\phi}}_{2,1} +\bold{v}_j +\frac{5}{4M^{2p+2}}\mathbf{1}, \bm{\hat{\phi}}_{2,1} + \bold{v}_j+\frac{2a}{M^2} \mathbf{1}-\frac{5}{4M^{2p+2}} \mathbf{1}, 1\right)\\
&\geq  1-\frac{M^d}{B_{M, \epsilon}} \geq 1- \frac{1}{4M^{2p+2+d}}
\end{align*}
and 
\begin{align*}
&\sigma\left(-B_1\left(\frac{1}{2} -\left(1-\sum_{j \in \{1, \dots, M^d\}} f_{test}\left(\hat{\phi}_{1,1}, \hat{\phi}_{2,1} +\bold{v}_j +\frac{5}{4M^{2p+2}}\mathbf{1}, \right.\right.\right.\right.\\
& \hspace*{5cm} \left. \left. \left. \left. \hat{\phi}_{2,1} + \bold{v}_j+\frac{2a}{M^2}\mathbf{1}-\frac{5}{4M^{2p+2}}\mathbf{1}, 1\right)\right)- B_2f_{id}^3\left(f_1(\x)\right)\right)\right)\\
&\geq \sigma\left(-B_1\left(\frac{1}{2} - \left(1-\frac{1}{4M^{2p+2+d}}\right)-B_2\left(1+\frac{1}{2M^{2p+2+d}}\right)\right)\right) 
\\& \geq \sigma\left(\frac{1}{4}B_1\right)  \geq 1-\frac{1}{M^{2p+2}}.
\end{align*}
In the third case we assume \eqref{check1}, but
\begin{align*}
\x \in \bigcup_{k \in \{1, \dots, M^{2d}\}} (C_{k,2})_{2/M^{2p+2}}^0.
\end{align*}
Then we have for all $k \in \{1, \dots, M^{2d}\}$
\begin{align*}
\left|x^{(s)} - (C_{k,2})_{left}^{(s)}\right| \geq \frac{2}{M^{2p+2}}
\end{align*}
and 
\begin{align*}
\left|x^{(s)} - (C_{k,2})_{left}^{(s)}-\frac{2a}{M^2}\right| \geq \frac{2}{M^{2p+2}}
\end{align*}
for $s \in \{1, \dots, d\}$. Together with the error bounds for $\bm{\hat{\phi}}_{1,1}$ and $\bm{\hat{\phi}}_{2,1}$ this leads to
\begin{align*}
\left|\hat{\phi}_{1,1}^{(s)} - \hat{\phi}_{2,1}^{(s)}-v_k^{(s)}\right| \geq \frac{2}{M^{2p+2}} - \frac{1}{4M^{2p+2}} -\frac{1}{6M^{2p+2}} \geq \frac{6}{4M^{2p+2}}
\end{align*}
and
\begin{align*}
\left|\hat{\phi}_{1,1}^{(s)} - \hat{\phi}_{2,1}^{(s)}-v_k^{(s)}-\frac{2a}{M^2}\right| \geq \frac{6}{4M^{2p+2}}
\end{align*}
for $k \in \{1, \dots, M^d\}$ and $s \in \{1, \dots, d\}$. Thus 
\begin{align*}
\mathds{1}_{\bigcup_{k \in \{1, \dots, M^{2d}\}} C_{k,2} \setminus (C_{k,2})_{2/(M^{2p+2})}^0}(\x) = 0
\end{align*}
implies 
\begin{align*}
\mathds{1}_{\bigcup_{k \in \{1, \dots, M^{d}\}} (\bar{\mathcal{A}}^{(k)})\setminus \bar{\mathcal{A}}^{(k)})_{6/(4M^{2p+2})}^0}(\bm{\hat{\phi}}_{1,1}) = 1-\sum_{k \in \{1, \dots, M^d\}} \mathds{1}_{(\bar{\mathcal{A}}^{(k)})_{6/(4M^{2p+2})}^0}(\bm{\hat{\phi}}_{1,1})=0
\end{align*}
and our network approximates the second indicator function since
\begin{align*}
&1-\sum_{j \in \{1, \dots, M^d\}} f_{test}\left(\bm{\hat{\phi}}_{1,1}, \bm{\hat{\phi}}_{2,1} +\bold{v}_j +\frac{5}{4M^{2p+2}}\mathbf{1}, \hat{\phi}_{2,1} + \bold{v}_j+\frac{2a}{M^2}\mathbf{1}-\frac{5}{4M^{2p+2}} \mathbf{1}, 1\right)\\
&\leq 1-\left(1-\frac{1}{B_{M, \epsilon}}\right) \leq \frac{1}{4M^{2p+2}}.
\end{align*}
and
\begin{align*}
&\sigma\left(-B_1\left(\frac{1}{2} - \left(1-\sum_{j \in \{1, \dots, M^d\}} f_{test}\left(\bm{\hat{\phi}}_{1,1}, \bm{\hat{\phi}}_{2,1} +\bold{v}_j +\frac{5}{4M^{2p+2}} \mathbf{1}, \right.\right.\right.\right.\\
& \left.\left.\left.\left. \hspace*{5.5cm} \bm{\hat{\phi}}_{2,1} + \bold{v}_j+\frac{2a}{M^2}  \mathbf{1}-\frac{5}{4M^{2p+2}}\mathbf{1}, 1\right)\right) - B_2f_{id}^3\left(f_1(\x)\right)\right)\right)\\
&\leq \sigma\left(-B_1\left(\frac{1}{2} -\frac{1}{4M^{2p+2}}+M^{2d} + \frac{1}{2M^{2p+2}}\right)\right) \leq \sigma\left(-\frac{1}{4}B_1\right) \leq \frac{1}{M^{2p+2}}.
\end{align*}
By construction of the network 
\begin{align*}
f_{check, \mathcal{P}_{2}}(\x) \in [0,1]
\end{align*}
holds for $\x \in [-a,a)^d$. 
\end{proof}

The following lemma approximates the function in \eqref{eq1000} by a DNN. Here we use the network $f_{check, \P_2}$ of \autoref{le11} for the indicator function and the network $f_{ReLU}$ of \autoref{le8} to approximate $g(z)=z_+$ $(z \in \R)$. The hidden layers of the networks are synchronized by using the network $f_{id}$ of \autoref{le1}. This network 
 approximately equals the output of $f(\x)$ in case that $\x \in \bigcup_{k \in \{1, \dots, M^{2d}\}} (C_{k,2})_{2/M^{2p+2}}^0$
and  approximately vanishes for $\x \in \bigcup_{k \in \{1, \dots, M^{2d}\}} C_{k,2}\textbackslash (C_{k,2})_{1/M^{2p+2}}^0$. 
\begin{lemma}
\label{le11}
Let $\sigma: \R \to [0,1]$ be the sigmoid activation function $\sigma(x) = 1/(1+\exp(-x))$. Let $1 \leq a < \infty$ and $M \in \N_0$ sufficiently large (at least $M \geq 3$). Let $p=q+s$ for some $q \in \N_0$, $s \in (0,1]$ and let $C>0$.
    Let $f: \Rd \to \R$ be a $(p,C)$-smooth function and let $f_{net, \P_2}$ be the network of \autoref{le7}. Then there exists a neural network 
\begin{align*}
f_{net, \P_2, true} \in \mathcal{F}(7+\lceil \log_2(q+1)\rceil, r, \alpha_{net, \P_2, true})
\end{align*}
with
\begin{align*}
r=\max\left\{\left(\binom{d+q}{d} + d\right) M^d(2+2d)+d, 4(q+1)\binom{d+q}{d}\right\} +M^d(2d+2)
\end{align*}
and 
\begin{align*}
\alpha_{net, \P_2, true}=c_{104} \left(\max\left\{a, \|f\|_{C^q([-a,a]^d)} \right\}\right)^{12} e^{6 \times 2^{2(d+1)+1}ad} M^{10p+2d+10}
\end{align*}
such that
\begin{alignat*}{3}
&|f_{net, \P_2, true}(\x)| \leq \frac{1}{M^{2p+2}} \quad &&\hspace{-2cm} \mbox{for} \ \x \in \bigcup_{k \in \{1, \dots, M^{2d}\}} C_{k,2}\textbackslash (C_{k,2})^0_{1/M^{2p+2}}\\
&|f_{net, \P_2,true}(\x) - f(\x)| \leq \frac{2c_{2}\left(\max\left\{a, \|f\|_{C^q([-a,a]^d)} \right\}\right)^{5q+3}}{M^{2p}}&&\\
\quad &&& \hspace{-2cm}\mbox{for} \ \x \in \bigcup_{k \in \{1, \dots, M^{2d}\}} (C_{k,2})_{2/M^{2p+2}}^0\\
&|f_{net, \P_2, true}(\x)| \leq |f_{net, \P_2}(\x)|+1 \quad &&\hspace{-2cm} \mbox{for} \ \x \in [-a,a]^d.
\end{alignat*}
\end{lemma}

\begin{proof} The proof is divided into \textit{two} steps.\\
\textit{Step 1: Network architecture:} Let $f_{id}$ be the network of \autoref{le1}, $f_{net, \mathcal{P}_{2}}$ be the network of \autoref{le7}, $f_{ReLU}$ be the network of \autoref{le8}
  and $f_{check, \mathcal{P}_{2}}$ be the network of \autoref{le10}. To combine the networks $f_{net, \mathcal{P}_{2}}$
  and $f_{check, \mathcal{P}_{2}}$ in a parallelized network we have to synchronize their number of hidden layers. In case that
  $q=0$ this can be done by applying $f_{id}$ to the network $f_{net, \mathcal{P}_{2}}$, otherwise we apply $5+\lceil \log_2(q+1) \rceil -6$-times the network
  $f_{id}$ to the network output of $f_{check, \mathcal{P}_{2}}$. In the following we assume w.l.o.g. $q > 0$ and set
\begin{eqnarray*}
  f_{net, \mathcal{P}_{2}, true}(\x) &=& f_{ReLU}\left(f_{net, \mathcal{P}_{2}}(\x) - 2B_{true}f_{id}^{5+\lceil \log_2(q+1) \rceil -6} \left(f_{check, \mathcal{P}_{2}}(\x)\right)\right)\\
 && + f_{ReLU}\left(-f_{net, \mathcal{P}_{2}}(\x) - 2B_{true}f_{id}^{5+\lceil \log_2(q+1) \rceil -6}\left(f_{check, \mathcal{P}_{2}}(\x)\right)\right),
\end{eqnarray*}
where 
\begin{align*}
B_{true} = 2^{2(d+1)}e^{2^{2(d+1)+1}ad} \max\{\|f\|_{C^q([-a,a]^d)},1\}+1.
\end{align*}
According to \autoref{le8} $f_{ReLU}$ satisfies
\begin{align}
\label{le12eq1}
|f_{ReLU}(x) - \max\{x,0\}| \leq \frac{1}{3M^{2p+2}}
\end{align}
for $x \in [-4B_{true}, 4B_{true}]$ (here we choose 
\begin{align*}
R\geq 208
\max \left\{
\| \sigma^{\prime \prime}\|_{\infty},
\| \sigma^{\prime \prime \prime}\|_{\infty},1
\right\}
 (4B_{true})^3 3M^{2p+2}
\end{align*}
 in \autoref{le8}). Furthermore we have 
\begin{align}
\label{le11eq2}
| f^t_{id}(x)-x|
\leq
\frac{1}{3M^{2p+2}B_{true}}, 
\end{align}
for $x \in [-1,1]$ and $t=5+\lceil \log_2(q+1) \rceil -6$ (where we choose $t_{\sigma, id} = 1$ and \linebreak $R \geq 2(5+\lceil \log_2(q+1) \rceil -7) \| \sigma^{\prime \prime}\|_{\infty}/(3M^{2p+2}B_{true}))$ in \autoref{le1}). In case that 
\begin{align}
\label{noset100}
\x \in \bigcup_{k \in \{1, \dots, M^{2d}\}}
    C_{k,2} \setminus (C_{k,2})_{1/M^{2p+2}}^0,
\end{align} 
we can bound the value of $f_{check, \mathcal{P}_{2}}$ by
\begin{align*}
1-\frac{1}{3M^{2p+2} B_{true}} \leq f_{check, \mathcal{P}_{2}}(\x) \leq 1
\end{align*}
according to \autoref{le10} (where we substitute $M^{2p+2}$ by $3M^{2p+2}B_{true}$). Since $f_{check, \mathcal{P}_{2}}(\x)$ is contained in the interval where \eqref{le11eq2} holds we then have
\begin{align*}
1-\frac{1}{3M^{2p+2}B_{true}}-\frac{1}{3M^{2p+2}B_{true}} \leq f_{id}^{5+\lceil \log_2(q+1) \rceil -6}(f_{check, \mathcal{P}_{2}}(\x)) \leq 1+\frac{1}{3M^{2p+2}B_{true}}.
\end{align*}
Since $f_{net, \P_2}(\x)$ is bounded in absolute values by $B_{true}$ according to \autoref{le7}, we have
\begin{align*}
-4B_{true} \leq f_{net, \mathcal{P}_{2}}(\x) - 2B_{true}f_{id}^{5+\lceil \log_2(q+1) \rceil -6} \left(f_{check, \mathcal{P}_{2}}(\x)\right) \leq 0
\end{align*}
and 
\begin{align*}
-4B_{true} \leq f_{net, \mathcal{P}_{2}}(\x) - 2B_{true}f_{id}^{5+\lceil \log_2(q+1) \rceil -6} \left(f_{check, \mathcal{P}_{2}}(\x)\right) \leq 0.
\end{align*}
Thus both networks are contained in the interval, where \eqref{le12eq1} holds. We can conclude that
\begin{align*}
\left|f_{net, \P_2, true}(\x)\right| &= \left|f_{ReLU}\left(f_{net, \mathcal{P}_{2}}(\x) - 2B_{true}f_{id}^{5+\lceil \log_2(q+1) \rceil -6} \left(f_{check, \mathcal{P}_{2}}(\x)\right)\right)\right.\\
& \left. \quad + f_{ReLU}\left(-f_{net, \mathcal{P}_{2}}(\x) - 2B_{true}f_{id}^{5+\lceil \log_2(q+1) \rceil -6}\left(f_{check, \mathcal{P}_{2}}(\x)\right)\right)\right|\\
& \leq  \left|f_{ReLU}\left(f_{net, \mathcal{P}_{2}}(\x) - 2B_{true}f_{id}^{5+\lceil \log_2(q+1) \rceil -6} \left(f_{check, \mathcal{P}_{2}}(\x)\right)\right) \right.\\
& \left. \quad \quad - \left(f_{net, \mathcal{P}_{2}}(\x)  - 2B_{true}f_{id}^{5+\lceil \log_2(q+1) \rceil -6} \left(f_{check, \mathcal{P}_{2}}(\x)\right)\right)_+\right|\\
& \quad + \left|f_{ReLU}\left(-f_{net, \mathcal{P}_{2}}(\x) - 2B_{true}f_{id}^{5+\lceil \log_2(q+1) \rceil -6} \left(f_{check, \mathcal{P}_{2}}(\x)\right)\right) \right.\\
& \left. \quad \quad - \left(-f_{net, \mathcal{P}_{2}}(\x)  - 2B_{true}f_{id}^{5+\lceil \log_2(q+1) \rceil -6} \left(f_{check, \mathcal{P}_{2}}(\x)\right)\right)_+\right|\\
&\quad +\left|\left(f_{net, \mathcal{P}_{2}}(\x)  - 2B_{true}f_{id}^{5+\lceil \log_2(q+1) \rceil -6} \left(f_{check, \mathcal{P}_{2}}(\x)\right)\right)_+\right|\\
&\quad +\left|\left(-f_{net, \mathcal{P}_{2}}(\x)  - 2B_{true}f_{id}^{5+\lceil \log_2(q+1) \rceil -6} \left(f_{check, \mathcal{P}_{2}}(\x)\right)\right)_+\right|\\
& \leq \frac{1}{3M^{2p+2}}+\frac{1}{3M^{2p+2}} + 0 +0 \leq \frac{1}{M^{2p+2}}.
\end{align*}
For 
\begin{align*}
\x \in \bigcup_{k \in \{1, \dots, M^{2d}\}} (C_{k,2})_{2/M^{2p+2}}^0
\end{align*}
the value of $f_{check, \P_2}(\x)$ is contained in $[0,1/(3M^{2p+2}B_{true})]$. Thus $f_{check, \P_2}(\x)$ is contained in the interval, where \eqref{le11eq2} holds. Assume w.l.o.g. $f_{net, \mathcal{P}_{2}}(\x) \geq 0$. This implies
\begin{align*}
-\frac{1}{3M^{2p+2}B_{true}} \leq f_{id}^{5+\lceil \log_2(q+1) \rceil -6}(f_{check, \mathcal{P}_{2}}(\x)) &\leq \frac{1}{3M^{2p+2B_{true}}}+\frac{1}{3M^{2p+2}B_{true}}\\
&= \frac{2}{3M^{2p+2}B_{true}}
\end{align*}
and 
\begin{align*}
0 \leq f_{net, \mathcal{P}_{2}}(\x) - 2B_{true} f_{id}^{5+\lceil \log_2(q+1) \rceil -6} \left(f_{check, \mathcal{P}_{2}}(\x)\right) \leq B_{true} +\frac{2}{3M^{2p+2}}
\end{align*}
and
\begin{align*}
 -f_{net, \mathcal{P}_{2}}(\x) - 2B_{true} f_{id}^{5+\lceil \log_2(q+1) \rceil -6} \left(f_{check, \mathcal{P}_{2}}(\x)\right) \leq -B_{true}-1.
\end{align*}
Again we can conclude that all inputs of $f_{ReLU}$ in the definition of $f_{net, \P_2, true}$ are contained in the interval, where \eqref{le11eq2} holds. Furthermore we have
\begin{align*}
\left|f_{net, \P_2}(\x) - f(\x)\right| \leq c_{2} \frac{ \left(\max\left\{a, \|f\|_{C^q([-a,a]^d)}\right\}\right)^{5q+3}}{M^{2p}}
\end{align*}
according to \autoref{le7}.
Thus we have
\begin{align*}
&\left|f_{net, \P_2, true}(\x) - f_{net, \P_2}(\x)\right|\\  
& \leq  \left|f_{ReLU}\left(f_{net, \mathcal{P}_{2}}(\x) - 2B_{true} f_{id}^{5+\lceil \log_2(q+1) \rceil -6} \left(f_{check, \mathcal{P}_{2}}(\x)\right)\right) \right.\\
& \left. \quad \quad - \left(f_{net, \mathcal{P}_{2}}(\x)  - 2B_{true} f_{id}^{5+\lceil \log_2(q+1) \rceil -6} \left(f_{check, \mathcal{P}_{2}}(\x)\right)\right)_+\right|\\
& \quad + \left|f_{ReLU}\left(-f_{net, \mathcal{P}_{2}}(\x) - 2B_{true} f_{id}^{5+\lceil \log_2(q+1) \rceil -6} \left(f_{check, \mathcal{P}_{2}}(\x)\right)\right) \right.\\
& \left. \quad \quad - \left(-f_{net, \mathcal{P}_{2}}(\x)  - 2B_{true} f_{id}^{5+\lceil \log_2(q+1) \rceil -6} \left(f_{check, \mathcal{P}_{2}}(\x)\right)\right)_+\right|\\
&\quad +\left|\left(f_{net, \mathcal{P}_{2}}(\x)  - 2B_{true} f_{id}^{5+\lceil \log_2(q+1) \rceil -6} \left(f_{check, \mathcal{P}_{2}}(\x)\right)\right)_+-f_{net, \P_2}(\x)\right|\\
& \quad + \left|f_{net, \P_2}(\x) - f(\x)\right|\\
&\quad +\left|\left(-f_{net, \mathcal{P}_{2}}(\x)  - 2B_{true} f_{id}^{5+\lceil \log_2(q+1) \rceil -6} \left(f_{check, \mathcal{P}_{2}}(\x)\right)\right)_+\right|\\
& \leq \frac{1}{3M^{2p+2}} + \frac{1}{3M^{2p+2}} + \frac{1}{3M^{2p+2}} +c_{2}\frac{ \left(\max\left\{a, \|f\|_{C^q([-a,a]^d)}\right\}\right)^{5q+3}}{M^{2p}} +0\\
& \leq 2c_{2} \frac{\left(\max\left\{a, \|f\|_{C^q([-a,a]^d)} \right\}\right)^{5q+3}}{M^{2p}}
\end{align*}
With the same argumentation we can conclude that 
\begin{align*}
\left|f_{net, \P_2, true}(\x) - f_{net, \P_2}(\x)\right| \leq 2c_{2} \frac{\left(\max\left\{a, \|f\|_{C^q([-a,a]^d)}\right\}\right)^{5q+3}}{M^{2p}}
\end{align*}
for $f_{net, \P_2}(\x) <0$. Since $f_{check, \P_2}(\x) \in [0,1]$ for $\x \in [-a,a]^d$ and by \eqref{le11eq2} 
\begin{align*}
f_{id}^{5+\lceil \log_2(q+1) \rceil -6}(f_{check, \P_2}(\x)) \in \left[-1/(3B_{true}M^{2p+2}), 1+1/(3B_{true}M^{2p+2})\right]
\end{align*}
 and $f_{net, \P_2}(\x) \in [-B_{true}, B_{true}]$ the inputs of $f_{ReLU}$ in the definition of $f_{net, \P_2, true}$ are contained in the interval where \eqref{le12eq1} holds. Using triangle inequality this leads to the bound
\begin{align*}
|f_{net, \mathcal{P}_2, true}(\x)| &\leq |f_{ReLU}\left(f_{net, \mathcal{P}_{2}}(\x) - 2B_{true} f_{id}^{5+\lceil \log_2(q+1) \rceil -6} \left(f_{check, \mathcal{P}_{2}}(\x)\right)\right) \\
& \hspace*{3cm} - \left(f_{net, \mathcal{P}_{2}}(\x) - 2B_{true} f_{id}^{5+\lceil \log_2(q+1) \rceil -6} \left(f_{check, \mathcal{P}_{2}}(\x)\right)\right)_+|\\
& \quad + |f_{ReLU}\left(-f_{net, \mathcal{P}_{2}}(\x) - 2B_{true}f_{id}^{5+\lceil \log_2(q+1) \rceil -6} \left(f_{check, \mathcal{P}_{2}}(\x)\right)\right) \\
& \hspace*{3cm} - \left(-f_{net, \mathcal{P}_{2}}(\x) - 2B_{true} f_{id}^{5+\lceil \log_2(q+1) \rceil -6} \left(f_{check, \mathcal{P}_{2}}(\x)\right)\right)_+|\\
&\quad + |\left(f_{net, \mathcal{P}_{2}}(\x) - 2B_{true} f_{id}^{5+\lceil \log_2(q+1) \rceil -6} \left(f_{check, \mathcal{P}_{2}}(\x)\right)\right)_+|\\
& \quad + |\left(-f_{net, \mathcal{P}_{2}}(\x) - 2B_{true}  f_{id}^{5+\lceil \log_2(q+1) \rceil -6} \left(f_{check, \mathcal{P}_{2}}(\x)\right)\right)_+|\\
& \leq \frac{1}{3M^{2p+2}} + \frac{1}{3M^{2p+2}} + 0 + f_{net, \P_2}(\x)+0 \leq |f_{net, \P_2}(\x)|+1
\end{align*}
for $\x \in [-a,a]^d$. 
\\
\\
The weights of the network $f_{net, \P_2, true}$ are bounded by
\begin{align*}
\alpha_{net, \P_2, true} &= \max\{\alpha_{net, \P_2} \alpha_{ReLU}, 2B_{true}\alpha_{ReLU}\alpha_{id}\}\\
&= c_{42} \left(\max\left\{a, \|f\|_{C^q([-a,a]^d)} \right\}\right)^{12} e^{6 \times 2^{2(d+1)+1} ad} M^{10p+2d+10},
\end{align*}
where we choose $R$ in the definition of $\alpha_{ReLU}$ and $\alpha_{id}$ as mentioned above.
\end{proof}

In the following lemma we combine the networks of \autoref{le9} and \autoref{le11} to approximate $w_{\P_2}(\x) f(\x)$ in supremum norm.
\begin{lemma}
\label{le12}
Let $\sigma: \R \to [0,1]$ be the sigmoid activation function $\sigma(x) = 1/(1+\exp(-x))$. Let $1 \leq a < \infty$ and $M \in \N_0$ sufficiently large (independent of the size of $a$, but 
    \begin{align*}
      M^{2p} \geq
     2c_{2} \left(\max\left\{a, \|f\|_{C^q([-a,a]^d)} \right\}\right)^{5q+3} \ \mbox{and} \ M^{2p} \geq  \max\left\{c_{3}, 2^d-1, \frac{4d}{a}\right\} 
    \end{align*}
    must hold).
Let $p=q+s$ for some $q \in \N_0$, $s \in (0,1]$ and let $C>0$.
    Let $f: \Rd \to \R$ be a $(p,C)$-smooth function and let $w_{\P_2}$ be defined as in
\eqref{w_vb}. Then there exists a network
\begin{align*}
f_{net} \in \mathcal{F}\left(L, r, \alpha_{net} \right)
\end{align*}
with
\begin{itemize}
\item[(i)] $L=8+\lceil \log_2(\max\{d, q+1\})\rceil$
\item[(ii)] $r=\max\left\{\left(\binom{d+q}{d} + d\right) M^d (2+2d)+d, 4 (q+1) \binom{d+q}{d}\right\} +M^d(2d+2)+12d$
\item[(iii)] $\alpha_{net}=c_{43} \left(\max\left\{a, \|f\|_{C^q([-a,a]^d)}\right\} \right)^{12} e^{6 \times 2^{2(d+1)+1} ad} M^{10p+2d+10}$
\end{itemize}
such that
\begin{align*}
&\left|f_{net}(\x) - w_{\P_2}(\x) f(\x)\right| \leq \frac{c_{44} \left(\max\left\{a, \ \|f\|_{C^q([-a,a]^d)}\right\}\right)^{5q+3}}{M^{2p}}
\end{align*}
for $\x \in [-a,a)^d$. 
\end{lemma}

\begin{proof}[\rm{\textbf{Proof of \autoref{le12}}}]
Using \autoref{le9}, \autoref{le11} and \autoref{le2} this proof follows as a straightforward modification from the proof of Lemma 12 in the Supplement of \cite{KL20}. A complete proof can be found in the Supplement. 
\end{proof}

%
\subsection{Key step 4: Applying $f_{net}$ to slightly shifted partitions}
In our last step we apply the networks of \autoref{le12} to $2^d$ slightly shifted versions of $\P_2$ and combine those
networks in a finite sum. The proof of \autoref{th1} follows in a straightforward way from the proof of Theorem 2 in \cite{KL20}, where we use the networks $f_{net,1}, \dots, f_{net, 2^d}$ of \autoref{le12}. A complete proof is given in the Supplement. 
%
%
%
%
%
%
%
%
%

\bibliographystyle{acm}
\bibliography{Literatur}
\newpage
\section*{Supplement}
\begin{proof}[\rm{\textbf{Proof of \autoref{le3}}}]
The proof can be divided into \textit{two} steps. \\
\textit{Step 1: Approximation of a monomial:} We will construct a neural network $f_m$, that approximates a monomial
\begin{equation*}
y m(\x) = y \prod_{k=1}^d (x^{(k)})^{r_k}, \quad \x \in [-a, a]^d, y \in [-a,a], 
\end{equation*}
where $m \in \mathcal{P}_N$ and $r_1, \dots, r_d \in \N_0$ satisfy $r_1+ \dots + r_d \leq N$.
To do this, we will use the idea from the proof of Lemma 5 in the Supplement of \cite{KL20}.
Set $q= \left\lceil \log_2\left( N+1\right)\right\rceil$. The feedforward neural network $f_m$  with $L=q$ hidden layers and $r=4(N+1)$ neurons in each layer is constructed as follows:
Set
\begin{equation*}
(z_1, \dots, z_{2^q})=
  \left(y, \underbrace{x^{(1)}, \dots, x^{(1)}}_{r_1}, \underbrace{x^{(2)}, \dots, x^{(2)}}_{r_2}, \dots, \underbrace{x^{(d)}, \dots, x^{(d)}}_{r_d}, \underbrace{1, \dots,1}_{2^q-\sum_{i=1}^d r_i-1} \right)
\end{equation*}
and let $f_{mult}$ be the network of \autoref{le2}.
%
In the first layer we compute
\[
f_{mult}(z_1,z_2), 
f_{mult}(z_3,z_4), 
\dots,
f_{mult}(z_{2^q-1},z_{2^q}), 
\]
which can be done by one layer of $4\times 2^{q-1} \leq 4(N+1)$
neurons. E.g., in case
$z_l=z_{l+1}=x^{(1)}$ we have
\[
f_{mult}(z_l,z_{l+1})
=f_{mult}(x^{(1)},x^{(1)})
\]
or in case $z_l=x^{(d)}$ and $z_{l+1}=1$ we have 
\[
f_{mult}(z_l,z_{l+1})
=f_{mult}(x^{(d)},1).
\]
As a result of this first layer we get a vector of outputs
which has length $2^{q-1}$. Next we pair these outputs and apply $f_{mult}$ again. This procedure is continued until there is only one output left.
Therefore we need $L =q$ hidden layers and
at most $4(N+1)$
neurons in each layer. As the resulting network is a combined and parallelized network, where the output of $f_{mult}$ is the input of $f_{mult}$ in the next hidden layer, we ``melt''  several times the weights of the previous output and the next input layer of the networks. Thus we have 
\[
  |c_{ij}^{(\ell)}| \leq 9R^4
\]
for $\ell \in \{0, \dots, q\}$ and $j\geq 0$ according to \autoref{le2}.
\\
\\
According to \autoref{le2} $f_{mult}$ satisfies
\begin{equation}
  \label{ple4eq1}
|f_{mult}(x,y) - x y| \leq 75\|
\sigma^{\prime \prime \prime}
\|_{\infty}
 \frac{(4^{N+1} a^{N+1})^3}{R} 
\end{equation}
for $x,y \in [-4^{N+1} a^{N+1},4^{N+1} a^{N+1}]$.
By  (\ref{ple4eq1}) 
and $ R \geq 75\|
\sigma^{\prime \prime \prime}
\|_{\infty}4^{3(N+1)} a^{3(N+1)}$
we get for any $\ell \in \{1,\dots,N\}$ and any
$z_1,z_2 \in [-(4^{\ell}-1) a^{\ell},(4^{\ell}-1) a^{\ell}]$
\[
|f_{mult}(z_1,z_2)| \leq
|z_1 z_2| + |f_{mult}(z_1,z_2)-z_1 z_2|
\leq
(4^{\ell}-1)^2 a^{2\ell} + 1
\leq
(4^{2\ell}-1) a^{2\ell}.
\]
From this we get successively that all outputs
of 
layer $\ell \in \{1,\dots,q-1\}$
  are contained in the interval
$[-(4^{2^{\ell}}-1) a^{2^{\ell}},(4^{2^{\ell}}-1) a^{2^{\ell}}]$, hence in particular they
  are contained in the interval 
$[-4^{N+1} a^{N+1},4^{N+1} a^{N+1}]$
where inequality  (\ref{ple4eq1}) does hold.
\\
\\
Define $f_{2^q}$ recursively by
\[
f_{2^q}(z_1,\dots,z_{2^q})
=
f_{mult}(f_{2^{q-1}}(z_1,\dots,z_{2^{q-1}}),f_{2^{q-1}}(z_{2^{q-1}+1},\dots,z_{2^q}))
\]
and
\[
f_2(z_1,z_{2})= f_{mult}(z_1,z_{2}),
\]
and set
\[
\Delta_{\ell}=\sup_{z_1,\dots,z_{2^{\ell}}
  \in [-a,a]}
\left|f_{2^{\ell}}(z_1,\dots,z_{2^{\ell}})-  \prod_{i=1}^{2^{\ell}} z_i\right|.
\]
Then
\[
|f_m(\x,y)-
m(\x,y)|
=
|f_m(\x,y)-
y\prod_{k=1}^d (x^{(k)})^{r_k}
|
\leq
\Delta_q
\]
and from
\[
\Delta_1 \leq
75\|
\sigma^{\prime \prime \prime}
\|_{\infty} \frac{4^{3(N+1)} a^{3(N+1)}}{R},
\]
which follows from (\ref{ple4eq1}) and
\begin{eqnarray*}
  && \Delta_q\\
  &&
  \leq
  \sup_{z_1,\dots,z_{2^q}
  \in [-a,a]}
  \left|f_{mult}(f_{2^{q-1}}(z_1,\dots,z_{2^{q-1}}),f_{2^{q-1}}(z_{2^{q-1}+1},\dots,z_{2^q}))\right.
  \\
  &&\left. \hspace*{5cm}
    -
    f_{2^{q-1}}(z_1,\dots,z_{2^{q-1}}) f_{2^{q-1}}(z_{2^{q-1}+1},\dots,z_{2^q})\right|
      \\
      &&
      \quad
      +
  \sup_{z_1,\dots,z_{2^q}
  \in [-a,a]}
      \left|f_{2^{q-1}}(z_1,\dots,z_{2^{q-1}}) f_{2^{q-1}}(z_{2^{q-1}+1},\dots,z_{2^q})\right.
      \\
      && \left. \hspace*{6.5cm}
      -
       \left( \prod_{i=1}^{2^{q-1}} z_i \right) f_{2^{q-1}}(z_{2^{q-1}+1},\dots,z_{2^q})\right|
      \\
      &&
      \quad
      +
  \sup_{z_1,\dots,z_{2^q}
  \in [-a,a]}
      \left|
        \left( \prod_{i=1}^{2^{q-1}} z_i \right)
        f_{2^{q-1}}(z_{2^{q-1}+1},\dots,z_{2^q})
          -
          \left( \prod_{i=1}^{2^{q-1}} z_i \right)
          \prod_{i=2^{q-1}+1}^{2^{q}} z_i
          \right|
          \\
          &&
          \leq
          75\|
\sigma^{\prime \prime \prime}
\|_{\infty} \frac{4^{3(N+1)} a^{3(N+1)}}{R}
          +
          2 \times
          4^{2^{q-1}} a^{2^{q-1}}
          \Delta_{q-1},
  \end{eqnarray*}
where  the last inequality follows from
(\ref{ple4eq1})
and the fact that all outputs of 
layer \linebreak $\ell \in \{1,\dots,q-1\}$
  are contained in the interval
$[-4^{2^{\ell}} a^{2^{\ell}},4^{2^{\ell}} a^{2^{\ell}}]$,
we get
for $\x \in [-a,a]^d$
\begin{align}
\label{le3eq2}
  |f_m(\x,y) - m(\x,y)|
  \leq \Delta_q
 & \leq
75\|
\sigma^{\prime \prime \prime}
\|_{\infty} \frac{4^{3(N+1)} a^{3(N+1)}}{R} 4^{2N} a^{2N} 2N \notag\\
&  \leq 
  150\|
\sigma^{\prime \prime \prime}
\|_{\infty} N 4^{5N+3} a^{5N+3} \frac{1}{R}.
  \end{align}
\textit{Step 2: Approximation of a polynomial:} By using the representation \eqref{le4p}, we can conclude 
\begin{align*}
&\left|p(\x, y_1, \dots, y_{\binom{d+N}{d}}) - \sum_{i=1}^{\binom{d+N}{d}} r_i f_{m_i}(\x, y_1, \dots, y_{\binom{d+N}{d}})\right|\\
&\leq \sum_{i=1}^{\binom{d+N}{d}} |r_i| \left|m_i(\x, y) - f_{m_i}(\x, y) \right|\\
&\leq
\binom{d+N}{d} \bar{r}(p)
 150\|
\sigma^{\prime \prime \prime}
\|_{\infty} N 4^{5N+3} a^{5N+3} \frac{1}{R}.
\end{align*}
Thus we have a neural network
\begin{align*}
f_{p}\left(\x, y_1, \dots, y_{\binom{d+N}{d}}\right) = \sum_{i=1}^{\binom{d+N}{d}} r_i f_{m_i}\left(\x, y_1, \dots, y_{\binom{d+N}{d}}\right) \in \mathcal{F}\left(L, 4(N+1) \binom{d+N}{d}, \alpha \right)
\end{align*}
with $\alpha = 9\max\{\bar{r}(p),1\}R^4$.
\end{proof}

\begin{proof}[\rm{\textbf{Proof of \autoref{le4}}}]
a)
Let $\kappa > 0$ be arbitrary. It is easy to see that the sigmoid activation function satisfies
\begin{align}
\label{indeq1}
\sigma(x) \geq 1 - \kappa \ \mbox{if} \ x \geq \ln\left(\frac{1}{\kappa} -1\right) \ \mbox{and} \ \sigma(x) \leq \kappa \ \mbox{if} \ x \leq -\ln\left(\frac{1}{\kappa} -1\right).
\end{align}
Using this together with the definition for $B_2$ we get for any 
\begin{align*}
\x \in [a^{(1)} + \delta, b^{(1)} - \delta] \times \dots \times [a^{(d)} + \delta, b^{(d)} - \delta]
\end{align*}
and $k \in \{1, \dots, d\}$ 
\begin{align*}
\sigma\left(B_2 (a^{(k)} - x^{(k)})\right) &= \sigma\left(\ln(4-1) \frac{a^{(k)} - x^{(k)}}{\delta}\right) \leq \sigma\left(-\ln\left(\frac{1}{\frac{1}{4}} -1\right)\right)\\
&= \frac{1}{1+e^{\ln(4-1)}} = \frac{1}{1+(4-1)} = \frac{1}{4}
\end{align*}
and
\begin{align*}
\sigma\left(B_2(x^{(k)} - b^{(k)})\right) = \sigma\left(\ln(4-1) \frac{x^{(k)} - b^{(k)}}{\delta}\right) \leq \sigma\left(-\ln\left(\frac{1}{\frac{1}{4}} -1\right)\right) = \frac{1}{4}.
\end{align*}
This implies
\begin{align*}
-B_1 \left(\sum_{k=1}^d\left(\sigma\left(B_2(a^{(k)}-x^{(k)})\right)+\sigma\left(B_2(x^{(k)}-b^{(k)})\right)\right)- \frac{5}{8} d\right) \geq \frac{1}{8}dB_1.
\end{align*}
Using the definition of $B_1$ and \eqref{indeq1} we get
\begin{align*}
f_{ind, [\bold{a}, \bold{b})}(\x) \geq \sigma\left(\frac{1}{8}dB_1\right) = \sigma\left(\ln\left(\frac{1}{\epsilon}-1\right)\right) \geq 1-\epsilon
\end{align*}
and therefore
\begin{align*}
|\mathds{1}_{[\bold{a}, \bold{b})}(\x) - f_{ind, [\bold{a}, \bold{b})}(\x)| = |1-f_{ind, [\bold{a}, \bold{b})}(\x)| \leq \epsilon.
\end{align*}
For $\x \notin [a^{(1)} - \delta, b^{(1)} + \delta] \times \dots \times [a^{(d)} - \delta, b^{(d)} + \delta]$ we know w.l.o.g. that there is a $k \in \{1, \dots, d\}$ which satisfies
\begin{align*}
x^{(k)} < a^{(k)} - \delta.
\end{align*}
Using \eqref{indeq1} and the definition of $B_2$ we can argue similar as above, that
\begin{align*}
\sigma\left(B_2(a^{(k)} - x^{(k)})\right) > \sigma\left(B_2\delta\right) = \sigma\left(\ln(3)\frac{1}{\delta}\delta\right) = \sigma(\ln(4-1)) \geq 1-\frac{1}{4} = \frac{3}{4}
\end{align*}
and therefore (since $\sigma(x) > 0$ for $x \in \R$) 
\begin{align*}
\sum_{k=1}^d \left(\sigma(B_2(a^{(k)} - x^{(k)})) + \sigma(B_2(x^{(k)} - b^{(k)}))\right) \geq \frac{3}{4}d. 
\end{align*}
This implies 
\begin{align*}
-B_1\left(\sum_{k=1}^d\left(\sigma\left(B_2(a^{(k)} - x^{(k)})\right)+\sigma\left(B_2(x^{(k)} - b^{(k)}\right)\right) - \frac{5}{8} d\right) \leq -\frac{1}{8} dB_1.
\end{align*}
Using the definition of $B_1$ and \eqref{indeq1} we get
\begin{align*}
f_{ind, [\bold{a}, \bold{b})}(\x) \leq \sigma\left(-\frac{1}{8} dB_1\right) = \sigma\left(-\ln\left(\frac{1}{\epsilon}-1\right)\right) = \frac{1}{1+\frac{1}{\epsilon}-1} = \epsilon, 
\end{align*}
which leads to 
\begin{align*}
|\mathds{1}_{[\bold{a}, \bold{b})}(\x) - f_{ind, [\bold{a}, \bold{b})}(\x)| = |-f_{ind, [\bold{a}, \bold{b})}(\x)| \leq \epsilon.
\end{align*}
For $\x \in \Rd$ it follows by the definition of the network
\begin{align*}
0 \leq f_{ind, [\bold{a}, \bold{b})}(\x) \leq 1.
\end{align*}
b) Let $f_{id}$ be the network of \autoref{le1} which satisfies
  \[
\left|f_{id}(x^{(i)})-x^{(i)}\right| \leq \frac{\epsilon}{3} \quad \mbox{for } x^{(i)} \in [-2R,2R], i \in \{1, \dots, d\}
\]
(so we choose $R_{id}= 3|\sigma'(t_{\sigma,id})|/ (\Vert \sigma''\Vert_{\infty} 2R^2 \epsilon)$ in \autoref{le1})
  and $f_{ind, [\bold{a}, \bold{b})}$ be the network of \autoref{le4} a), which satisfies
  \begin{align*}
  \left|f_{ind, [\bold{a}, \bold{b})}(\x)-
  \mathds{1}_{[\bold{a}, \bold{b})}(\x)\right| \leq \frac{\epsilon}{3R}
\end{align*}
for $\x \in K_{\delta}$ (so we choose $\epsilon = \epsilon/3R$ in the definition of $B_1$). Furthermore let
  $f_{mult}$ be the network of \autoref{le2} which satisfies
  \[
  |f_{mult}(x,y) - x y|
  \leq
  \frac{\epsilon}{3}
  \quad \mbox{for } x,y \in [-2R,2R]
  \]
  (so we choose $R_{mult} = 1800\Vert \sigma'''\Vert_{\infty} R^3/\epsilon$ in \autoref{le2}). By \autoref{le4} a) we have for $\x \in \R^d$
\begin{align*}
|f_{ind, [\bold{a}, \bold{b})}(\x)| \leq 1 \leq R.
\end{align*}
Furthermore it follows for $s \in [-R,R]$ 
\begin{align*}
|f_{id}(s)| \leq |f_{id}(s)-s|+|s| \leq 2R
\end{align*}
holds. Using this together with the above inequalities we can conclude
  \begin{eqnarray*}
    |f_{test}(\x, \mathbf{a}, \mathbf{b}, s) -
  s \mathds{1}_{[\bold{a}, \bold{b})}(\x)|
    &
    =
    &
    \left|f_{mult}\left(f_{id}(s),f_{ind, [\bold{a}, \bold{b})}(\x)\right) - s \mathds{1}_{[\bold{a}, \bold{b})}(\x)\right|
      \\
      &
      \leq
      &
      \left|  f_{mult}\left(f_{id}(s),f_{ind, [\bold{a}, \bold{b})}(\x)\right) 
-
f_{id}(s) f_{ind, [\bold{a}, \bold{b})}(\x)
\right|
\\
&& +
\left|
f_{id}(s) f_{ind, [\bold{a}, \bold{b})}(\x)
      - s  f_{ind, [\bold{a}, \bold{b})}(\x)\right|
      \\
&& +
\left|
s f_{ind, [\bold{a}, \bold{b})}(\x)
- s  \mathds{1}_{[\bold{a}, \bold{b})}(\x)\right|
     \\
      &
     \leq
     &
     \frac{\epsilon}{3} + 1\frac{\epsilon}{3} + s\frac{\epsilon}{3R}
     = \epsilon
    \end{eqnarray*}
    for $\x \in K_{\delta}$ and $s \in [-R,R]$. With the same argumentation we can conclude that 
    \begin{align*}
     |f_{test}(\x, \mathbf{a}, \mathbf{b}, s) -
  s\mathds{1}_{[\bold{a}, \bold{b})}(\x)| \leq \frac{\epsilon}{3} + 1\frac{\epsilon}{3} + |s| \times 1 \leq 2|s|
    \end{align*}
    for $\x \in \Rd$. It is easy to see that the network is contained in the class $\mathcal{F}(3, 2+2d, \alpha)$ with the weights of the network bounded by 
    \begin{align*}
    \alpha = \max\left\{\alpha_{ind, [\bold{a}, \bold{b})}\frac{2}{R_{mult}}, c_{10}R_{id} \frac{2}{R_{mult}}, c_{20} R_{mult}^2, c_{10} R_{id}, \alpha_{ind, [\bold{a}, \bold{b})}\right\} = c_{14} \max\left\{\frac{R^6}{\epsilon^2}, \frac{1}{\delta}\right\}.
%
    \end{align*}
\end{proof}

\begin{proof}[\rm{\textbf{Proof of \autoref{le9}}}] The proof is divided into \textit{two} steps.\\ 
\textit{Step 1: Network architecture:} The first five hidden layers of $f_{w_{\P_2}}$ approximate the value of $(C_{\mathcal{P}_{2}}(\x))_{left}$
and shift the value of $\x$ in the next hidden layer, respectively.  This can be done as described in $\bm{\hat{\phi}}_{1, 2}$ and $\bm{\hat{\phi}}_{2,2}$ in the proof of \autoref{le7} with $d+M^d d(2+2d)$ neurons per layer, where all weights in the network are bounded by $\alpha_{test} \max\{\alpha_{id}, a\}$ (as defined in \autoref{le7}). The sixth and seventh hidden layer then compute the functions
\begin{eqnarray*}
w_{\P_2,j}(\x) &&=\left(1-\frac{M^2}{a}\left|(C_{\mathcal{P}_{2}}(\x))_{left}^{(j)} + \frac{a}{M^2} -x^{(j)}\right| \right)_+\\
  &&=
  \left(
  \frac{M^2}{a}
  \left(
x^{(j)} - (C_{\mathcal{P}_{2}}(\x))_{left}^{(j)} 
  \right)
  \right)_+
  \\
  &&
  \quad
  -
  2\left(
  \frac{M^2}{a}
  \left(
  x^{(j)} - (C_{\mathcal{P}_{2}}(\x))_{left}^{(j)}
  - \frac{a}{M^2}
  \right)
  \right)_+
 \\
  &&
  \quad
  +
  \left(
  \frac{M^2}{a}
  \left(
  x^{(j)} - (C_{\mathcal{P}_{2}}(\x))_{left}^{(j)}
  - \frac{2a}{M^2}
  \right)
  \right)_+, \quad j \in \{1, \dots, d\}.
\end{eqnarray*}
This can be done using the network
\begin{align*}
f_{ReLU} \in \mathcal{F}(2, 4, \alpha_{ReLU})
\end{align*}
with
\begin{align*}
\alpha_{ReLU} = c_{21}a^6M^{4p}
\end{align*}
from \autoref{le8}, 
which satisfies
\begin{align}
\label{le9eq1}
\left|f_{ReLU}(z) - \max\{z,0\}\right| \leq \frac{1}{M^{2p}}
\end{align}
for all $z \in [-2^{2(d+1)+1}aM^2,2^{2(d+1)+1}aM^2]$ (here we choose 

\begin{align*}
R=2^{6(d+1)}a^3M^{2p+6}(208
\max \left\{
\| \sigma^{\prime \prime}\|_{\infty},
\| \sigma^{\prime \prime \prime}\|_{\infty},1
\right\})
\end{align*}
 in \autoref{le8}). We set
\begin{align}
\label{eq300}
  f_{w_{{\P_2},j}}(\x) &= f_{ReLU}\left(
  \frac{M^2}{a}
  \left(
\hat{\phi}_{1,2}^{(j)} - \hat{\phi}_{2, 2}^{(j)} 
  \right)
  \right) -2f_{ReLU}\left(
\frac{M^2}{a}
  \left(
  \hat{\phi}_{1, 2}^{(j)} - \hat{\phi}_{2, 2}^{(j)}
  - \frac{a}{M^2}
  \right) 
  \right)\notag \\
& \quad + f_{ReLU}\left(
  \frac{M^2}{a}
  \left(
  \hat{\phi}_{1, 2}^{(j)} - \hat{\phi}_{2, 2}^{(j)}
  - \frac{2a}{M^2}
  \right)
  \right).
\end{align}
The product of $w_{\P_2,j}(\x)$ $(j \in \{1, \dots, d\})$
can then be computed by a network $f_m$ described in the first step of the proof of \autoref{le3}, where we choose $y=1$, $r_1 = \dots = r_d =1$ and $N+1=d$, $x^{(j)} = f_{w_{{\P_2},j}}(\x)$ and  $R=M^{2p}$ with weights bounded by $\alpha_m = c_{22}M^{4p}$. Finally we set
\begin{align*}
f_{w_{\P_2}}(\x) = f_m\left(f_{w_{{\P_2},1}}(\x), \dots, f_{w_{{\P_2},d}}(\x)\right).
\end{align*}
This network lies in the class
\begin{align*}
\mathcal{F}\left(5+2+\lceil\log_2(d)\rceil, \max\{12d, d+M^dd(2+2d)\}, \alpha_{w_{\P_2}}\right)
\end{align*}
with
\begin{align*}
\alpha_{w_{\P_2}} = \max\{\alpha_{test} \alpha_{id}, \alpha_{test}\alpha_{ReLU}, \alpha_{ReLU}\alpha_m\} = c_{3}\left(\max\left\{a, \|f\|_{C^q([-a,a]^d)}\right\}\right)^7 M^{6p+6+2d}.
\end{align*}
\\
\\
\textit{Step 2: Approximation error:}
In case that
\begin{align*}
\x \in \bigcup_{k \in \{1, \dots, M^{2d}\}} \left(C_{k,2}\right)_{1/M^{2p+2}}^0
\end{align*}
we can bound the value of $|\hat{\phi}_{1, 2}^{(j)}-\hat{\phi}_{2,2}^{(j)}|$ $(j \in \{1, \dots, d\})$ using \eqref{fs1} by $3a$, such that the inputs of $f_{ReLU}$ in \eqref{eq300} are contained in the interval where \eqref{le9eq1} holds. Thus we have
\begin{align*}
&\left|f_{w_{\P_2, j}}(\x) - w_{\P_2, j}(\x)\right|\\
 &\leq \left|f_{ReLU}\left(
  \frac{M^2}{a}
  \left(
\hat{\phi}_{1,2}^{(j)} - \hat{\phi}_{2, 2}^{(j)} 
  \right)
  \right) - \left(
  \frac{M^2}{a}
  \left(
\hat{\phi}_{1,2}^{(j)} - \hat{\phi}_{2, 2}^{(j)} 
  \right)
  \right)_+\right|\\
  & \quad + 2 \left|f_{ReLU}\left(
  \frac{M^2}{a}
  \left(
\hat{\phi}_{1,2}^{(j)} - \hat{\phi}_{2, 2}^{(j)} 
 -\frac{a}{M^2}\right)
 \right) - \left(
  \frac{M^2}{a}
  \left(
\hat{\phi}_{1,2}^{(j)} - \hat{\phi}_{2, 2}^{(j)} 
  - \frac{a}{M^2}\right) 
  \right)_+\right|\\
  & \quad + \left|f_{ReLU}\left(
  \frac{M^2}{a}
  \left(
\hat{\phi}_{1,2}^{(j)} - \hat{\phi}_{2, 2}^{(j)} 
  -\frac{2a}{M^2}\right)
  \right) - \left(
  \frac{M^2}{a}
  \left(
\hat{\phi}_{1,2}^{(j)} - \hat{\phi}_{2, 2}^{(j)} 
  -\frac{2a}{M^2}\right)
  \right)_+\right|\\
  &\quad +\left|\left(
  \frac{M^2}{a}
  \left(
\hat{\phi}_{1,2}^{(j)} - \hat{\phi}_{2, 2}^{(j)} 
  \right)
  \right)_+ - \left(
  \frac{M^2}{a}
  \left(
\phi_{1,2}^{(j)} - \phi_{2, 2}^{(j)} 
  \right)
  \right)_+\right|\\
  & \quad + 2 \left|\left(
  \frac{M^2}{a}
  \left(
\hat{\phi}_{1,2}^{(j)} - \hat{\phi}_{2, 2}^{(j)} 
 -\frac{a}{M^2}\right)
 \right)_+ - \left(
  \frac{M^2}{a}
  \left(
\phi_{1,2}^{(j)} - \phi_{2, 2}^{(j)} 
  - \frac{a}{M^2}\right) 
  \right)_+\right|\\
  & \quad + \left|\left(
  \frac{M^2}{a}
  \left(
\hat{\phi}_{1,2}^{(j)} - \hat{\phi}_{2, 2}^{(j)} 
  -\frac{2a}{M^2}\right)_+
  \right) - \left(
  \frac{M^2}{a}
  \left(
\phi_{1,2}^{(j)} - \phi_{2, 2}^{(j)} 
  -\frac{2a}{M^2}\right)
  \right)_+\right|\\
  & \leq \frac{4}{M^{2p}} + \frac{8}{M^{2p} a} \leq \frac{12}{M^{2p}}, 
\end{align*}
where we used for the last inequality that $\max\{x,0\}$ is Lipschitz continuous, that 
\begin{align*}
|\hat{\phi}^{(j)}_{1,2} - \phi_{1,2}^{(j)}|\leq \frac{1}{4M^{2p+2}}, \quad (j \in \{1, \dots, d\})
\end{align*}
according to \eqref{le7eq5} and that
\begin{align*}
|\hat{\phi}^{(j)}_{2,2} - \phi_{2,2}^{(j)}| \leq \frac{1}{2M^{2p+2}}, \quad (j \in \{1, \dots, d\})
\end{align*}
according to \eqref{le7eq16}. Since $M^{2p} \geq 12$ we can bound the value of each network $f_{w_{{\P_2},j}}$ by
\begin{align*}
\left|f_{w_{{\P_2},j}}(\x) - w_{{\P_2},j}(\x)\right|+\left|w_{{\P_2},j}(\x)\right| \leq 2,
\end{align*}
where we used that $\left|w_{{\P_2},j}(\x)\right| \leq 1$.
According to \eqref{le3eq2} (where we set $a=2$ and $N+1=d$) $f_m$ approximates the product of its input components with an error of size
\begin{align}
\label{w10}
\frac{c_{23} (d-1) 4^{5d-2} 2^{5d-2} }{M^{2p}} \leq \frac{c_{3}}{M^{2p}}.
\end{align}
We set
\begin{align*}
  \bar{f}_{w_{{\P_2},j}}(\x) &= \left(
  \frac{M^2}{a}
  \left(
\hat{\phi}_{1, 2}^{(j)} - \hat{\phi}_{2, 2}^{(j)} 
  \right)
  \right)_+ -2\left(
\frac{M^2}{a}
  \left(
  \hat{\phi}_{1, 2}^{(j)} - \hat{\phi}_{2, 2}^{(j)}
  - \frac{a}{M^2}
  \right) 
  \right)_+ \\
 & \quad  + \left(
  \frac{M^2}{a}
  \left(
  \hat{\phi}_{1, 2}^{(j)} - \hat{\phi}_{2, 2}^{(j)}
  - \frac{2a}{M^2}
  \right)
  \right)_+
\end{align*}
for $j \in \{1, \dots, d\}$. Using this and the above mentioned results we can bound the error of our network $f_{w_{\P_2}}$ by
\begin{align*}
\left|f_{w_{\P_2}}(\x) - w_{\P_2}(\x)\right| &\leq \left|f_m\left(f_{w_{{\P_2},1}}(\x), \dots, f_{w_{{\P_2},d}}(\x)\right) - \prod_{j=1}^d f_{w_{{\P_2},j}}(\x)\right|\\
& \quad + \left|\prod_{j=1}^d f_{w_{{\P_2},j}}(\x) - \bar{f}_{w_{\P,1}}(\x) \prod_{j=2}^d f_{w_{{\P_2},j}}(\x)\right|\\
& \quad + \dots \\
& \quad +\left|\prod_{j=1}^{d-1} \bar{f}_{w_{\P_2,j}}(\x) f_{w_{{\P_2},d}}(\x) - \prod_{j=1}^{d} \bar{f}_{w_{\P,j}}(\x)\right|\\
& \quad + \left|\prod_{j=1}^d \bar{f}_{w_{{\P_2},j}}(\x) - w_{\P_2,1}(\x) \prod_{j=2}^d \bar{f}_{w_{{\P_2},j}}(\x)\right|\\
& \quad + \dots \\
& \quad +\left|\prod_{j=1}^{d-1} w_{\P_2,j}(x) \bar{f}_{w_{{\P_2},d}}(\x) - \prod_{j=1}^{d} w_{\P_2,j}(\x)\right|\\
& \leq \frac{c_{3}}{M^{2p}}+\frac{2^{d-1}}{M^{2p}} + \dots + \frac{2}{M^{2p}}+d\frac{12}{M^{2p}}\\
& \leq \frac{\max\left\{c_{3}, 2^d, 12d\right\}}{M^{2p}}
\end{align*}
where we have used that $|\bar{f}_{w_{{\P_2},j}}(\x)| \leq 1$ and $|w_{\P_2,j}(\x)| \leq 1$ $(j \in \{1, \dots, d\})$. 
\\
\\
According to \autoref{le7} the values of $\bm{\hat{\phi}}_{1, 2}$ and $\bm{\hat{\phi}}_{2,2}$ are bounded by 
\begin{align*}
|\hat{\phi}_{1, 2}^{(s)}| \leq 2a \ \mbox{and} \ |\hat{\phi}_{2,2}^{(s)}| \leq 2^{2(d+1)} a
\end{align*}
for $s \in \{1, \dots, d\}$
in case that 
\begin{align*}
\x \notin \bigcup_{k \in \{1, \dots, M^{2d}\}} \left(C_{k,2}\right)_{1/M^{2p+2}}^0.
\end{align*}
Thus we have
\begin{align*}
|\hat{\phi}_{1, 2}^{(s)} - \hat{\phi}_{2,2}^{(s)}| \leq 2^{2(d+1)+1} a
\end{align*}
and the input of $f_{ReLU}$ in the definition of $f_{w_{\P_2}, j}$ lies in the interval, where \eqref{le9eq1} holds. Using triangle inequality, we can then bound 
\begin{align*}
|f_{w_{\P_2,j}}(\x)|  \leq |f_{w_{\P_2,j}}(\x) - \bar{f}_{w_{\P_2,j}}(\x)| + |\bar{f}_{w_{\P_2,j}}(\x)| \leq 2,
\end{align*}
where we use that $|\bar{f}_{w_{\P_2,j}}(\x)| \leq 1$ for $j \in \{1, \dots, d\}$. Using \eqref{w10} and triangle inequality again, this leads to 
\begin{align*}
|f_{w_{\P_2}}(\x)| \leq \left|f_{w_{\P_2}}(\x) - \prod_{j=1}^d f_{w_{{\P_2},j}}(\x)\right| + \left|\prod_{j=1}^d f_{w_{{\P_2,j}}}(\x)\right| \leq 2^{d+1}
\end{align*}
for $\x \in [-a,a)^d$, where we have used that
\begin{align*}
M^{2p} \geq c_{23} (d-1) 4^{5d-2} 2^{5d-2}.
\end{align*}
\end{proof}
\begin{proof}[\rm{\textbf{Proof of \autoref{le12}}}] The proof is divided into \textit{two} steps.\\
\textit{Step 1: Network architecture:} Let $f_{w_{\P_2}}$ be the network of \autoref{le8} and $f_{net, \mathcal{P}_{2},true}$ be the network of \autoref{le11}. To multiply the network $f_{net, \mathcal{P}_{2},true}$ by $f_{w_{\P_2}}$ we use the network $f_{mult}$ from \autoref{le2}. Furthermore we use the network $f_{id}^t$ with $t \in \N$ from \autoref{le1} to synchronize the number of hidden layers of the networks $f_{net, \mathcal{P}_{2},true}$ and $f_{w_{\P_2}}$.
\\
\\
The final network is given by
\begin{align*}
f_{net}(\x) = f_{mult}\left(f_{id}^{t_1}(f_{w_{\P_2}}(\x)), f_{id}^{t_2}(f_{net, \mathcal{P}_{2}, true}(\x))\right)
\end{align*}
with
\begin{align*}
t_1 = \max\{\lceil \log_2(q+1)\rceil - \lceil\log_2(d)\rceil,0\} \ \mbox{and} \ t_2 = \max\{\lceil \log_2(d)\rceil - \lceil\log_2(q+1)\rceil,0\}.
\end{align*}
Now it is easy to see that this combined and parallelized network consists of 
\begin{align*}
L=8+\lceil \log_2(\max\{d, q+1\})\rceil
\end{align*}
and 
\begin{align*}
r=\max\left\{\left(\binom{d+q}{d} + d\right)M^d(2+2d)+d, 4(q+1)\binom{d+q}{d}\right\} +M^d(2d+2)+12d
\end{align*}
neurons per layer, where all weights in the network are bounded by
\begin{align*}
\max\{\alpha_{w_{\P_2}}, \alpha_{net, \P_2, true}, \alpha_{mult}\} = \alpha_{net, \P_2, true}.
\end{align*}
\textit{Step 2: Approximation error:} To analyze the approximation error of the network we use that
\begin{align}
\label{appfmult}
\left|f_{mult}(x,y) - xy\right| \leq \frac{1}{M^{2p}}
\end{align}
for all $x,y$ contained in
\begin{align*}
&\left[-2^{d+1}\max\left\{\|f\|_{\infty, [-a,a]^d},1\right\},  2^{d+1}\max\left\{\|f\|_{\infty, [-a,a]^d} ,1\right\}\right].
\end{align*}
Here we have chosen $R=75\|\sigma'''\|_{\infty}2^{3(d+1)}\left(\max\left\{\|f\|_{\infty, [-a,a]^d}, 1\right\}\right)^3 M^{2p}$ in \autoref{le2}. Furthermore we have
\begin{align}
\label{appfid}
\left|f_{id}^t(x) -x\right| \leq \frac{1}{M^{2p}}
\end{align}
for $x \in [-2 \max\{\max_{\x \in [-a,a]^d} f(\x), 1\}, 2 \max\{\max_{\x \in [-a,a]^d} f(\x), 1\}]$. Here we choose $t_{\sigma, id} = 0$ and 
\begin{align*}
R_{id} = (t-1)\|\sigma''\|_{\infty} 8 \max\{ \|f\|_{\infty, [-a,a]^d}, 1\}^2
\end{align*} in \autoref{le1}.
\\
\\
In case that
\begin{align*}
\x \in \bigcup_{k \in \{1, \dots, M^{2d}\}} \left(C_{k,2}\right)_{2/M^{2p+2}}^0,
\end{align*}
the value of $\x$ is neither contained in
\begin{align}
\label{noset1}
\bigcup_{k \in \{1, \dots, M^{2d}\}}
C_{k,2} \setminus (C_{k,2})^0_{1/M^{2p+2}}
\end{align}
nor contained in
\begin{align}
\label{noset2}
\bigcup_{k \in \{1, \dots, M^{2d}\}}
(C_{k,2})^0_{1/M^{2p+2}} \setminus (C_{k,2})^0_{2/M^{2p+2}}
.
\end{align}
Thus the network $f_{w_{\P_2}}(\x)$ approximates $w_{\P_2}(\x)$ with an error of size
\begin{align}
\label{appwv}
\frac{\max\left\{c_{3}, 2^d, 12d\right\}}{M^{2p}}
\end{align}
and $f_{net, \mathcal{P}_{2},true}(\x)$ approximates $f(\x)$ with an error of size
\begin{align}
\label{200}
\frac{2c_{2}\left(\max\left\{a, \|f\|_{C^q([-a,a]^d)} \right\}\right)^{5q+3}}{M^{2p}}.
\end{align}
Thus both networks are contained in the interval, where \eqref{appfid} holds. Here we use that
\begin{align*}
M^{2p} \geq \max\left\{c_{3}, 2^d, 12d\right\} 
\end{align*}
and 
\begin{align*}
M^{2p} \geq 2c_{2}\left(\max\left\{a, \|f\|_{C^{q}([-a,a]^d)}\right\}\right)^{5q+3}
\end{align*}
by assumption. Thus the value of $f_{id}^{t_1}(f_{w_{\P_2}}(\x))$ and $f_{id}^{t_2}(f_{net, \mathcal{P}_{2},true}(\x))$ are contained in the interval where \eqref{appfmult} holds. Using triangle inequality, this implies
\begin{align*}
&\left|f_{mult}\left(f_{id}^{t_1}(f_{w_{\P_2}}(\x)), f_{id}^{t_2}(f_{net, \mathcal{P}_{2}, true}(\x))\right) - w_{\P_2}(\x)f(\x)\right|\\
& \leq \left|f_{mult}\left(f_{id}^{t_1}(f_{w_{\P_2}}(\x)), f_{id}^{t_2}(f_{net, \mathcal{P}_{2}, true}(\x))\right) - f_{id}^{t_1}(f_{w_{\P_2}}(\x))f_{id}^{t_2}(f_{net, \mathcal{P}_{2}}(\x))\right|\\
& \quad +\left|f_{id}^{t_1}(f_{w_{\P_2}}(\x))f_{id}^{t_2}(f_{net, \mathcal{P}_{2}}(\x))-f_{w_{\P_2}}(\x)f_{id}^{t_2}(f_{net, \mathcal{P}_{2}}(\x))\right|\\
& \quad + \left|f_{w_{\P_2}}(\x)f_{id}^{t_2}(f_{net, \mathcal{P}_{2}}(\x)) - f_{w_{\P_2}}(\x)f_{net, \mathcal{P}_{2}}(\x)\right|\\
& \quad + \left|f_{w_{\P_2}}(\x)f_{net, \mathcal{P}_{2}}(\x) - w_{\mathcal{P}_{2}}(\x) f_{net, \mathcal{P}_{2}}(\x)\right|\\
  & \quad +
  \left|
  w_{\P_2}(\x)f_{net, \mathcal{P}_{2}}(\x) - w_{\P_2}(\x)f(\x)
  \right|\\
& \leq \frac{c_{24}\left(\max\left\{a, \|f\|_{C^q([-a,a]^d)} \right\}\right)^{5q+3}}{M^{2p}}.
\end{align*}
In case that $\x$ is contained in \eqref{noset1} the value of $f_{net, \mathcal{P}_{2}, true}$ is of size $1/M^{2p}$ according to \autoref{le11}. Furthermore we have 
\begin{align*}
\left|f_{w_{\P_2}}(\x)\right| &\leq 2^{d+1}.
\end{align*}
 Thus $f_{w_{\P_2}}(\x)$ and $f_{net, \mathcal{P}_{2}, true}$ are contained in the interval, where \eqref{appfid} holds and $f_{id}^{t_1}(f_{w_{\P_2}}(\x))$ and $f_{id}^{t_2}(f_{net, \mathcal{P}_{2}, true}(\x))$ are contained in the interval, where \eqref{appfmult} holds. Together with 
\begin{align*}
w_{\P_2}(\x) \leq \frac{1}{aM^{2p}}
\end{align*}
and the triangle inequality it follows 
\begin{align*}
&\left|f_{mult}\left(f_{w_{\P_2}}(\x), f_{net, \mathcal{P}_{2}, true}(\x)\right) - w_{\P_2}(\x)f(\x)\right|  \leq \frac{c_{25}\left(\max\{\|f\|_{\infty, [-a,a]^d} , 1\}\right)^2}{M^{2p}}.
\end{align*}
In case that $\x$ is in \eqref{noset2} but not in \eqref{noset1} the network $f_{net, \mathcal{P}_{2}}(\x)$ approximates $f(\x)$ with an error as in \eqref{200}. Furthermore, $f_{w_{\P_2}}(\x) \in [-2,2]$
approximates $w_{\P_2}(\x)$ with an error as in \eqref{appwv}.
The value of $f_{net, \P_2, true}$ is bounded by $|f_{net, \P_2}(\x)| +1$.
Hence $f_{w_{\P_2}}(\x)$ and $f_{net, \mathcal{P}_{2}, true}(\x)$ are contained in the interval, where \eqref{appfid} holds and $f_{id}^{t_1}(f_{w_{\P_2}}(\x))$ and $f_{id}^{t_2}(f_{net, \mathcal{P}_{2}, true}(\x))$ are contained in \eqref{appfmult}. Together with 
\begin{align*}
w_{\P_2}(\x) \leq \frac{2}{aM^{2p}}
\end{align*}
and the triangle inequality it follows again
\begin{align*}
&\left|f_{mult}\left(f_{w_{\P_2}}(\x), f_{net, \mathcal{P}_{2}, true}(\x)\right) - w_{\P_2}(\x)f(\x)\right| \leq \frac{c_{26} \left(\max\{\|f\|_{\infty, [-a,a]^d} , 1\}\right)^2}{M^{2p}}.
\end{align*} 
\end{proof}

\begin{proof}[\rm{\textbf{Proof of Theorem 1}}]
By increasing $a$, if necessary, it suffices to show that there exists a network $f_{net}$ satisfying
\begin{align*}
&\sup_{\x \in [-a/2,a/2]^d} \left|f(\x) - f_{net}(\x)\right| \leq \frac{c_{27}\left(\max\left\{a, \|f\|_{C^q([-a,a]^d)} \right\}\right)^{5q+3}}{M^{2p}}.
\end{align*}

Let $\mathcal{P}_1$ and $\mathcal{P}_2$ be the partitions defined as in \eqref{partition}. We set
\begin{align*}
\mathcal{P}_{1,1} = \mathcal{P}_1 \ \mbox{and} \ \mathcal{P}_{2,1} = \mathcal{P}_2
\end{align*}
and define for each $v \in \{2, \dots, 2^d\}$ partitions $\mathcal{P}_{1,v}$ and $\mathcal{P}_{2,v}$, which are modifications of $\mathcal{P}_{1,1}$ and $\mathcal{P}_{2,1}$ where at least one of the components it shifted by $a/M^2$. 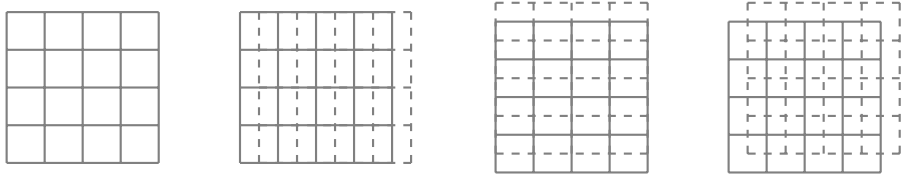
\begin{figure}[h!]
\centering
\begin{minipage}{0.20\textwidth}
\begin{tikzpicture}
\draw[step=0.5cm,color=gray, thick] (-1,-1) grid (1,1);
\end{tikzpicture}
\end{minipage}
\begin{minipage}{0.22\textwidth}
\begin{tikzpicture}
\draw[step=0.5cm,color=gray, thick] (-1,-1) grid (1,1);
\draw[step=0.5cm,color=gray, thick, xshift=0.25cm, dashed] (-1,-1) grid (1,1);
%
\end{tikzpicture}
\end{minipage}
\begin{minipage}[c]{0.20\textwidth}
\begin{tikzpicture}
\draw[step=0.5cm,color=gray, thick, yshift=0.5cm] (-1,-1) grid (1,1);
\draw[step=0.5cm,color=gray, thick, yshift=0.75cm, dashed] (-1,-1) grid (1,1);
%
\end{tikzpicture}
\end{minipage}
\begin{minipage}{0.20\textwidth}
\begin{tikzpicture}
\draw[step=0.5cm,color=gray, thick, yshift=0.25cm] (-1,-1) grid (1,1);
\draw[step=0.5cm,color=gray, thick, yshift=0.5cm, xshift=0.25cm, dashed] (-1,-1) grid (1,1);
\end{tikzpicture}
\end{minipage}
\caption{Shifted partitions for the case $d=2$}
\label{fig8}
\end{figure}
The idea is illustrated for the case $d=2$ in \hyperref[fig8]{Fig.\ref*{fig8} }. Here one sees, that
for $d=2$
there exist $2^2=4$ different partitions, if we shift our partition along at least one component by the same additional summand. We denote by $C_{k, 2,v}$ the corresponding cubes of the partition $\mathcal{P}_{2,v}$ $(k \in \{1, \dots, M^{2d}\})$.
\\
\\
The idea of the proof of \autoref{th1} is to compute a linear combination of networks $f_{net, \mathcal{P}_{2,1}}, \dots, f_{net, \mathcal{P}_{2,2^d}}$ of \autoref{le7} (where the $\mathcal{P}_{2,v}$ are treated as $\mathcal{P}_2$ in \autoref{le7}, respectively). 
To avoid that the approximation error of the networks increases close to the border of some cube of the partitions, we multiply each value of $f_{net, \mathcal{P}_{2,v}}$ with a weight 
\begin{align}
\label{w_v}
w_v(\x) = \prod_{j=1}^d \left(1- \frac{M^2}{a}\left|(C_{\mathcal{P}_{2,v}}(x))_{left}^{(j)} + \frac{a}{M^2} - x^{(j)}\right|\right)_+.
\end{align}
It is easy to see that $w_v(\x)$ is a linear tensorproduct B-spline
which takes its maximum value at the center of $C_{\P_{2,v}}(\x)$, which
is nonzero in the inner part of $C_{\P_{2,v}}(\x)$ and which
vanishes
outside of $C_{\P_{2,v}}(\x)$. Consequently
we have \linebreak $w_1(\x)+ \dots + w_{2^d}(\x) = 1$
for $\x \in [-a/2,a/2]^d$.
Let $f_{net,1}, \dots, f_{net, 2^d}$ be the networks of \autoref{le12}
corresponding to the partitions
$\mathcal{P}_{1,v}$ and $\mathcal{P}_{2,v}$ $(v \in \{1, \dots, 2^d\})$, respectively. 
Since $[-a/2, a/2]^d \subset [-a+a/M^2, a+a/M^2)^d$ each $\P_{1,v}$ and $\P_{2,v}$ form a partition of a set which contains $[-a/2,a/2]^d$ and the error bounds of \autoref{le12} hold for each network $f_{net, v}$ on $[-a/2,a/2]^d$.
We set
\begin{align*}
f_{net}(\x) = \sum_{v=1}^{2^d} f_{net,v}(\x).
\end{align*}
Since
\begin{align*}
f(\x) = \sum_{v=1}^{2^d} w_v(\x)f(\x)
\end{align*}
it follows directly by \autoref{le12}
\begin{align*}
  \left|f_{net}(\x) - f(\x)\right| &= \left|\sum_{v=1}^{2^d} f_{mult}\left(f_{id}^{t_1}(f_{w_v}(\x)), f_{id}^{t_2}(f_{net, \mathcal{P}_{2,v}, true}(\x))\right) - \sum_{v=1}^{2^d} w_v(\x)f(\x)\right|\\
& \leq \sum_{v=1}^{2^d} \left|f_{mult}\left(f_{id}^{t_1}(f_{w_v}(\x)), f_{id}^{t_2}(f_{net, \mathcal{P}_{2,v}, true}(\x))\right) -  w_v(\x)f(\x)\right|\\
& \leq \frac{c_{28}\left(\max\left\{a, \|f\|_{C^q([-a,a]^d)} \right\}\right)^{5q+3}}{M^{2p}},
\end{align*}
where $t_1$ and $t_2$ are chosen as in the proof of \autoref{le12}.
\end{proof}
 \end{document}